\documentclass{article}

\usepackage[a4paper]{geometry}

\usepackage{booktabs}

\usepackage{tikz}

\usepackage[utf8]{inputenc}
\usepackage{cite}
\usepackage{amsmath,amssymb,amsfonts}

\usepackage{algorithm}
\usepackage{algorithmic}

\newcommand{\rw}[1]{\begingroup#1\endgroup}

\usepackage{graphicx}
\usepackage{textcomp}
\def\BibTeX{{\rm B\kern-.05em{\sc i\kern-.025em b}\kern-.08em
		T\kern-.1667em\lower.7ex\hbox{E}\kern-.125emX}}

\newcommand{\indep}{\rotatebox[origin=c]{90}{$\models$}}
\newcommand{\mat}[1]{\boldsymbol{\mathrm{#1}}}
\newcommand{\vecb}[1]{\boldsymbol{#1}}

\DeclareMathOperator{\tr}{tr}
\DeclareMathOperator{\pr}{pr}
\DeclareMathOperator{\pa}{pa}

\DeclareMathOperator{\var}{var}

\DeclareMathOperator{\mgfrob}{grad\_frob}
\DeclareMathOperator{\mglik}{grad\_lik}
\DeclareMathOperator{\mband}{band}
\DeclareMathOperator{\mlasso}{lasso}
\DeclareMathOperator{\fone}{F1}
\DeclareMathOperator{\tdr}{TDR}
\DeclareMathOperator{\tpr}{TPR}
\DeclareMathOperator{\tnr}{TNR}
\DeclareMathOperator{\acc}{ACC}
\DeclareMathOperator{\norm}{NORM}

\usepackage{amsthm}

\newtheorem*{proposition*}{Proposition}
\newtheorem{proposition}{Proposition}
\newtheorem{remark}{Remark}

\begin{document}
	\author{Irene Córdoba, Concha Bielza, Pedro Larrañaga and Gherardo Varando}

\title{Sparse Cholesky Covariance Parametrization for Recovering Latent Structure in Ordered Data}

\date{}

\maketitle

\begin{abstract}
The sparse Cholesky parametrization of the inverse covariance matrix is directly related to Gaussian Bayesian networks. Its counterpart, the covariance Cholesky factorization model, has a natural interpretation as a hidden variable model for ordered signal data. Despite this, it has received little attention so far, with few notable exceptions. To fill this gap, in this paper we \rw{focus on} arbitrary zero pattern\rw{s} in the Cholesky factor of a covariance matrix. \rw{We discuss how these models can also be extended, in analogy with Gaussian Bayesian networks, to data where no apparent order is available. For the ordered scenario, we propose a novel estimation method that is based on matrix loss penalization, as opposed to the existing regression-based approaches}. The performance of this \rw{sparse model for the Cholesky factor, together with our novel estimator,} is assessed in a simulation setting \rw{, as well as over} spatial and temporal real data where a natural ordering \rw{a}rises among the variables. We give guidelines, based on the empirical results, about which of the methods analysed is more appropriate for each setting.
\end{abstract}


\section{Introduction}
The multivariate Gaussian distribution is central in both statistics and
machine learning because of its wide applicability and well theoretical
behaviour. Whenever it is necessary to reduce model dimension, such as in a
high dimensional setting \cite{buehlmann2011} or in graphical models
\cite{maathuis2018,li2020}, sparsity is imposed in either the covariance matrix
$\mat{\Sigma}$ or its inverse $\mat{\Omega} = \mat{\Sigma}^{-1}$. 

The inverse covariance matrix
$\mat{\Omega}$ directly encodes information about partial correlations; therefore, when a zero pattern is present in $\mat{\Omega}$, it represents \rw{absent edges in} the undirected graph \rw{of} a
Gaussian Markov network \cite{dempster1972,yuan2007}. Furthermore, letting $\mat{\Omega} =
\mat{W}\mat{W}^t$ be its Cholesky decomposition, a zero pattern in the lower
triangular matrix $\mat{W}$ yields the acyclic digraph associated with a
Gaussian Bayesian network \cite{wright1934,wermuth1980} model, up to a permutation of the variables \cite{vandegeer2013}. As a result, much of the
academic focus has been on sparsity in either the inverse covariance matrix or its Cholesky decomposition
\cite{pourahmadi1999,daspermont2008,friedman2008,rothman2008,cordoba2020b}.

Conversely, a zero pattern in the covariance matrix $\mat{\Sigma}$ \rw{represents missing edges from the undirected graph of a} covariance graph \rw{model}
\cite{kauermann1996,cox1993,chaudhuri2007,bien2011}. However, a
structured zero pattern on 
the Cholesky decomposition $\mat{\Sigma} = \mat{T}\mat{T}^t$ of the covariance matrix
has only been addressed by
few works\rw{:} \cite{wermuth2006} \rw{and} \cite{rothman2010}. In \cite{wermuth2006} the authors
briefly analyse zeros in $\mat{T}$ mainly as a tool for better understanding of
a higher-level graphical model called covariance chain, which is the main
focus of their work. More directly related \rw{to} our interests is the work of
Rothman et al. \cite{rothman2010}, who directly explore a regression
interpretation of the Cholesky factor $\mat{T}$ over the \emph{error}
variables. 
However, the authors do not address a
structured zero pattern 
directly; instead, they focus on a banding
structure for $\mat{T}$, inspired by the popularity of this estimator for the
covariance matrix. 
In fact, a significant amount of the paper is
devoted to analysing the relationship between 
the covariance matrix, or its inverse, and banded Cholesky factorization.

To fill this gap, in this paper we focus on arbitrary zero patterns in the
Cholesky factor $\mat{T}$ of the covariance matrix $\mat{\Sigma}$. We argue how this naturally models scenarios with ordered variables representing noisy inputs from hidden signal sources, see for example \cite{das2019} where latent brain regions are assumed to influence neural measurements. We discuss how this model could be extended to unordered variables by defining a new Gaussian graphical model over the Cholesky factorization of the covariance matrix. This new graphical model can be thought of as the analogue of a Gaussian Bayesian network, which is defined via sparse factorizations of the \emph{inverse} covariance matrix. For ordered variables, we propose a novel learning method by \rw{directly penalising a matrix loss, in contrast with existing regression-based approaches in the literature \cite{rothman2010}. This new estimator} is feasible \rw{to compute} thanks to a simplification of the \rw{loss} gradient computation similar to \rw{the one proposed in} \cite{varando2020}. Finally, we empirically assess the model performance in a
broad range of experiments, exploring multiple simulation scenarios as
well as real data where a natural spatial or temporal order arises between the
variables.

The rest of the paper is organized as follows. In Section \ref{sec:prel} we
introduce the theoretical preliminaries necessary to follow the exposition.
Then we detail the proposed sparse model for the Cholesky factor $\mat{T}$
of the covariance matrix $\mat{\Sigma}$\rw{, as well as introduce its extension to a Gaussian graphical model,} in Section \ref{sec:graph}. Afterwards,
we discuss \rw{existing state-of-the-art estimation} methods in Section
\ref{sec:learn}, where we also detail \rw{our novel penalized matrix loss proposal}. We
assess \rw{the previously explained estimation methods} in both simulated and real
experiments\rw{,} whose results are shown in Section
\ref{sec:exp}. Finally, we close the paper in Section \ref{sec:conc},
discussing the conclusions that can be \rw{drawn} from the presented work, and
also outlining the planned lines of future research. \rw{Appendices \ref{app:gbn:rel}, \ref{app:gradient} and \ref{app:fig} are referenced throughout the paper and contain additional material for the interested reader.}

\section{Theoretical preliminaries}\label{sec:prel}

We will denote as $\mathcal{N}(\vecb{\mu}, \mat{\Sigma})$ the multivariate
Gaussian distribution with mean vector $\vecb{\mu}$ and covariance matrix
$\mat{\Sigma}$. Since we focus on the zero pattern in $\mat{\Sigma}$, in the
following we will set $\vecb{\mu} = \vecb{0}$ for notational simplicity,
without loosing generality. As in the previous section, we will denote the
inverse covariance matrix as $\mat{\Omega} = \mat{\Sigma}^{-1}$.

\subsection{The inverse covariance matrix and systems of recursive regressions}

As we stated previously, the inverse covariance matrix and its decompositions
have been thoroughly researched. If $\vecb{X}$ is a $p$-variate Gaussian random
vector, that is, if $\vecb{X} \sim \mathcal{N}(\vecb{0}, \mat{\Sigma})$, then the
\emph{upper} Cholesky factorization of $\mat{\Omega}$ can be used to model
ordered sequences of regressions \cite{wermuth1980,pourahmadi1999}, as follows.

For  $J \subseteq \{1, \ldots, p\}$, $i \not\in J$ and $j \in J$, \rw{let} $\vecb{\beta}_{i |J} = (\rw{\mat{\Sigma}}_{iJ}\mat{\Sigma}^{-1}_{JJ})^{\rw{t}}$ \rw{denote} the vector of regression coefficients of $X_i$ on
$\vecb{X}_J$, with its $j$-th entry being $\beta_{ij | J}$, the coefficient corresponding to variable $X_j$. We may equivalently write the statistical model for $\vecb{X}$ as
a system of recursive linear regression equations 
(also called a linear structural equation model \cite{drton2018}),
where for each $i \in \{1,
\ldots, p\}$, \begin{equation}\label{eq:receq} X_{i} = \sum_{j = 1}^{i - 1}
\beta_{ij | 1, \ldots, i - 1}X_{j} + \mathcal{E}_{i}, \end{equation} with
$\mathcal{E}_i$ independent Gaussian random variables of zero mean \rw{and $\var(\mathcal{E}_i) = \var(X_i | X_1, \ldots, X_{i - 1})$}. The above
regression system can also be expressed as $\vecb{X} =
\mat{B}\vecb{X} + \vecb{\mathcal{E}}$, \rw{where} $\mat{B}$ \rw{is a} strictly lower triangular \rw{matrix}
with entries $b_{ij} = \beta_{ij | 1, \ldots, i - 1}$.  Rearranging the
previous equation, we obtain \begin{equation}\label{eq:latent} \vecb{X} =
\mat{L}\vecb{\mathcal{E}}, \end{equation} where $\mat{L} = (\mat{I}_p -
\mat{B})^{-1}$ with $\mat{I}_p$ the $p$-dimensional identity matrix.  Now if we
take variances and inverses on Equation \eqref{eq:latent}, we arrive at the
upper Cholesky decomposition of the inverse covariance matrix,
\begin{equation}\label{eq:chol:inv} \mat{\Omega} = \mat{\Sigma}^{-1} =
	\mat{L}^{-t} \mat{D}^{-1} \mat{L}^{-1} = \mat{U} \mat{D}^{-1} \mat{U}^t
	= \mat{W}\mat{W}^t, \end{equation} 
where $\mat{U} = \mat{L}^{-t} =	(\mat{I}_p - \mat{B})^t$ and $\mat{W} = \mat{U}\sqrt{\mat{D}^{-1}}$ are	upper triangular \rw{matrices, and $\mat{D}$ is a diagonal matrix containing the variances of $\vecb{\mathcal{E}}$}.

Equations \eqref{eq:receq} and \eqref{eq:chol:inv} are intimately related: all
parameters of the recursive regression system are encoded in the upper Cholesky
factorization. Indeed, we have that
$\mat{D} = \var(\vecb{\mathcal{E}})$ \rw{and} the $(j, i)$ entry of matrix
$\mat{U}$, $u_{ji}$\rw{,} is equal to $-\beta_{ij | 1, \ldots, i - 1}$, therefore all
the regression coefficients can also be recovered, and $\mat{U}$ is explicitly written as
\begin{equation}\label{eq:u}
\mat{U} = \begin{pmatrix}
1 & -\beta_{21 | 1} & \cdots & -\beta_{p1 | 1, \ldots, p - 1}\\
0 & 1 & \ddots & \vdots\\
\vdots & \ddots & \ddots & -\beta_{p\rw{p - 1} | 1, \ldots, p - 1}\\
0 & \cdots & 0 & 1\\
\end{pmatrix}.
\end{equation}

\subsection{The Cholesky decomposition of a covariance matrix}
From Equation \eqref{eq:latent} we could have proceeded differently. Instead of taking variances and then inverses, most suited for analysing ordered sequences of variable regressions and subsequent Gaussian Bayesian networks (GBNs), we could have just taken variances and would have arrived at the Cholesky decomposition of the covariance matrix \cite{rothman2010}
\begin{equation}\label{eq:chol:cov}
\mat{\Sigma} = \mat{L}\mat{D}\mat{L}^t = \mat{T}\mat{T}^t,
\end{equation}
where now $\mat{L} = (\mat{I}_p - \mat{B})^{-1}$ and $\mat{T} = \mat{L}\sqrt{\mat{D}}$ are lower triangular. 

Observe that Equation \eqref{eq:chol:cov} is a direct analogue of Equation \eqref{eq:chol:inv}; \rw{although} now there is not a natural interpretation of $\mat{L}$ entries in terms of regression coefficients. However, we will now show that the entries of $\mat{L}$ are in fact regression coefficients, just over a different conditioning set. Alternative derivations of this result can be found in \cite[p. 158]{dempster1969} and \cite[p. 846]{wermuth2006}. \rw{The one in \cite{dempster1969} is computational, based on the sweep matrix operator \cite{beaton1964}, whereas \cite{wermuth2006} provides a sketch based on a recursive expression for regression coefficients. We use instead simple identities over partitioned matrices.}
\begin{proposition}\label{prop:invreg}
	For $i \in \{1, \ldots, p\}$ and $j < i$, the
	$(i, j)$ entry of matrix $\mat{L}$, denoted as $l_{ij}$, is equal to
	${\beta}_{ij | 1, \ldots, j}$.
\end{proposition}
\begin{proof}
	For each $i \in \{1, \ldots, p\}$, 
	$j < i$ 
	and $J = \{1, \ldots, j\}$, the following
	partitioned identities hold \cite[Equation (4.2.18)]{dempster1969}
	\begin{align*}
	&\rw{\mat{\Sigma}}_{iJ} = \rw{\mat{L}}_{iJ}\mat{D}_{JJ} \mat{L}_{JJ}^t,\\
	&\mat{\Sigma}_{JJ} = \mat{L}_{JJ}\mat{D}_{JJ} \mat{L}_{JJ}^t.
	\end{align*}
	Therefore,
	\[
	\begin{aligned}
	\vecb{\beta}_{i|J}^{\rw{t}} &= \rw{\mat{\Sigma}}_{iJ} \mat{\Sigma}_{JJ}^{-1}\\ &= \rw{\mat{L}}_{iJ}\mat{D}_{JJ}
	\mat{L}_{JJ}^t(\mat{L}_{JJ}
	\mat{D}_{JJ} \mat{L}_{JJ}^t)^{-1}\\
	&= \rw{\mat{L}}_{iJ}\mat{L}_{JJ}^{-1}.
	\end{aligned}
	\]
	
Furthermore, observe that, since $\mat{L}_{JJ}$ is lower triangular with ones
along the diagonal, the last column of $\mat{L}_{JJ}^{-1}$ is always a vector
of zero entries except the last \rw{entry,} which is $1$. This means, in particular,
that for each $i \in \{1, \ldots, p\}$, $j < i$ and $J = \{1, \ldots,
j\}$, the $j$-th element of \rw{row} vector $\rw{\mat{L}}_{iJ}\mat{L}_{JJ}^{-1}$ is equal to
$l_{ij}$, which in turn is equal to the $j$-th entry of
$\vecb{\beta}_{i|J}$, $\beta_{ij|J}$.
\end{proof}

The explicit expression in terms of regression coefficients for $\mat{L}$ 
is therefore,
\begin{equation}\label{eq:l}
\mat{L} = \begin{pmatrix}
1 & 0 & \cdots & 0\\
{\beta}_{21 | 1} & 1 & \ddots & \vdots\\
\vdots & \ddots & \ddots & 0\\
{\beta}_{p 1 | 1} & \cdots &  {\beta}_{p p - 1 | 1, \ldots, p - 1} & 1\\
\end{pmatrix}.
\end{equation}
From Equations \eqref{eq:u} and \eqref{eq:l} 
useful relationships can be determined; see, for example, \cite[p.847]{wermuth2006} 
and Appendix \ref{app:gbn:rel}.

\subsection{The Gaussian Bayesian network model and sparse Cholesky factorizations of $\mat{\Omega}$}
The upper Cholesky factorization in Equation \eqref{eq:chol:inv} is typically
used as a parametrization for \rw{GBNs}, see, for
example, \cite{cordoba2020b,vandegeer2013,ye2020}. In the following we will explain
why, since this is closely related \rw{to} the extension of our proposal to
a graphical model. 

Let $\mathcal{G} = (V, E)$ be an acyclic digraph, where $V = \{1, \ldots, p\}$ \rw{is the vertex set}
and $E \subseteq V \times V$ \rw{is the edge set}. Assume that $1 \prec
\cdots \prec p$ is a \emph{topological order} of $V$, which means that for all $i \in V$, denoting as $\pa(i)$ the parent set in $\mathcal{G}$ of node $i$, it holds that $\pa(i) \subseteq \{1, \ldots, i - 1\}$. The ordered Markov property of Bayesian
networks states that \cite{wermuth1980}
\begin{equation}\label{eq:Uordered} 
X_i \indep X_j |\vecb{X}_{\pa(i)} \text{ for all } j < i \text{ and } j \not\in\pa(i)\rw{,} 
\end{equation}
\rw{where $X_i \indep X_j | \vecb{X}_{\pa(i)}$ stands for conditional independence \cite{dawid1979}, that is, $X_i$ and $X_j$ are independent given $\vecb{X}_{\pa(i)} = \vecb{x}_{\pa(i)}$, for any value of $\vecb{x}_{\pa(i)}$.}
In the multivariate Gaussian distribution, the above conditional independence is equivalent to the regression coefficient $\beta_{ij | 1, \ldots, i - 1}$ being zero (see for example \cite{anderson2003}).
Since $\beta_{ij | 1, \ldots, i - 1} = 0$ if and only if $u_{ji} = 0$, then the zero pattern containing absent edges in $\mathcal{G}$ can be read off from $\mat{U}$ in Equation \eqref{eq:chol:inv}. 

Now allow the ancestral order to be arbitrary, and denote as $\mathcal{U}(\mathcal{G})$ the set of matrices that have positive diagonal and zeros compatible with a given acyclic digraph $\mathcal{G} = (V, E)$; that is, such that if $(j, i) \not\in E$, \rw{$j \neq i$}, then $m_{ji} = 0$ for all $\mat{M} \in \mathcal{U}(\mathcal{G})$. The GBN model
can thus be expressed as
\begin{equation}\label{eq:gbn}
\mathcal{N}(\vecb{0}, \mat{\Sigma}) \text{ s.t. } {\mat{\Sigma}}^{-1}= \mat{W}\mat{W}^t
\text{ with } \mat{W} \in \mathcal{U}(G).
\end{equation}

\begin{remark}\label{rem:gbn:chol}
	Observe that if $\vecb{X} = (X_1, \ldots, X_p)^t \sim \mathcal{N}(\vecb{0}, \mat{\Sigma})$ follows a GBN model with graph $\mathcal{G}$, the parameter matrix $\mat{W}$ of Equation \eqref{eq:gbn} is not the Cholesky factor of $\mat{\Omega} = \mat{\Sigma}^{-1}$. This only occurs when the ancestral order of the nodes in $\mathcal{G}$ is $1 \prec \cdots \prec p$; that is, when the variables follow a natural order (the direct analogue of our model).
	
However, if we denote as $\kappa(\mat{M})$ the permutation of rows
and columns in $\mat{M}$ following the ancestral order of $\mathcal{G}$, then
$\kappa(\mat{W})$ is the upper Cholesky factor of $\kappa(\mat{\Omega})$.
\end{remark}

\section{Sparse Cholesky decomposition of the covariance matrix}\label{sec:graph}
We will hereby introduce the sparse Cholesky decomposition model $\mat{\Sigma} = \mat{T}\mat{T}^t$ for the covariance matrix, which consists \rw{of} allowing an arbitrary zero pattern in $\mat{T}$. The entries of $\mat{T}$ are in correspondence with those in $\mat{L}$ and $\mat{D}$ (see Equation \eqref{eq:chol:cov}, $t_{ij} = l_{ij}\sqrt{d_{ii}}$), and therefore sometimes we will indistinctly refer to each of these matrices .

\subsection{The regression interpretations for the covariance Cholesky factor}
In terms of model estimation, in $\mat{U}$ (Equation~\eqref{eq:u})  
each column corresponds to
the parameters of a single recursive regression, 
whereas each entry of matrix $\mat{L}$
(Equation~\eqref{eq:l}) corresponds to a different
regression model.

Fortunately, \cite{rothman2010} gave an alternative regression interpretation
for matrix $\mat{L}$: recalling Equation~\eqref{eq:latent}, where the original
variables $\vecb{X}$ are written as a linear function of the error terms
$\vecb{\mathcal{E}}$, and unfolding this matrix equation, 
we obtain 
\begin{equation}\label{eq:reg:cov} X_i = \sum_{j =
1}^{i - 1}l_{ij}\mathcal{E}_j + \mathcal{E}_i. \end{equation} 

Thus obtaining an
analogue of Equation \eqref{eq:receq}, but now instead of regressing the
original variables, the regression is performed over the \emph{error} terms of
the ordered recursive regressions in Equation \eqref{eq:receq}.

\begin{remark}\label{rem:reg}
	The $(i, j)$ entry of matrix $\mat{L}$ for $j < i$, $l_{ij}$, has therefore two interpretations as a regression coefficient:
	\begin{enumerate}
		\item It is the coefficient of the error $\mathcal{E}_j$ on the regression of $X_i$ over $\mathcal{E}_1, \ldots, \mathcal{E}_{i - 1}$.
		\item It is the coefficient of variable $X_j$ in the regression of $X_i$ over $X_1,\ldots,X_j$.
	\end{enumerate}
	Furthermore, from Equations \eqref{eq:receq} and \eqref{eq:reg:cov} we have another dual interpretation for variable $\mathcal{E}_i$: it is the error of the regression of $X_i$ onto $X_1, \ldots, X_{i - 1}$, but also of $X_i$ onto $\mathcal{E}_1, \ldots, \mathcal{E}_{i - 1}$.
\end{remark}

\subsection{A hidden variable model interpretation}
The sparse Cholesky parametrization of the covariance matrix naturally models 
a hidden variable structure \cite{chandrasekaran2012,yatsenko2015,zorzi2016,basu2019} over ordered Gaussian observables (Equation \eqref{eq:latent}). 
Interpreting the \emph{error} terms $\vecb{\mathcal{E}}$ as latent signal sources, then 
the model is a sort of restricted GBN. 
This interpretation is naturally associated with the first regression coefficients 
outlined in Remark \ref{rem:reg}.

The constraints for this GBN that represents our model are:
\begin{itemize}
	\item All arcs are from hidden variables $\vecb{\mathcal{E}}$ to the observed ones $\vecb{X}$.
	\item There is always an arc from $\mathcal{E}_i$ to $X_i$, for all $i \in \{1, \ldots, p\}$.
	\item For each $i \in \{1, \ldots, p\}$, only variables $\mathcal{E}_1, \ldots, \mathcal{E}_{i - 1}$ can have arcs to $X_i$.
\end{itemize}
Figure \ref{fig:rbm1} represents one such restricted GBN compatible with our model.
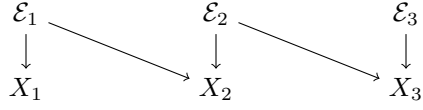
\begin{figure}[h]
	\centering
	\begin{tikzpicture}
	\node (h1) at (0,1) {$\mathcal{E}_1$};
	\node (h2) at (2.5,1) {$\mathcal{E}_2$};
	\node (h3) at (5,1) {$\mathcal{E}_3$};
	\node (x1) at (0,0) {$X_1$};
	\node (x2) at (2.5,0) {$X_2$};
	\node (x3) at (5,0) {$X_3$};
	\draw[->] (h1) -> (x1);
	\draw[->] (h1) -> (x2);
	\draw[->] (h2) -> (x2);
	\draw[->] (h2) -> (x3);
	\draw[->] (h3) -> (x3);
	\end{tikzpicture}
	\caption{Hidden variable model interpretation. We have omitted the index set notation for vertices and directly used variables for clarity.}
	\label{fig:rbm1}
\end{figure}

The sparse Cholesky factor for Figure \ref{fig:rbm1} would have the following lower triangular pattern
\begin{equation}
	\begin{pmatrix}
		 &  & \\
		* &  & \\
		0 & * & 
	\end{pmatrix},
\end{equation}
where an asterisk means a non-zero entry. \rw{Observe that the zero value in entry $(3, 1)$ corresponds to the missing edge from $\mathcal{E}_1$ to $X_3$ in Figure \ref{fig:rbm1}.}

\subsection{A graphical model extension for unordered variables}
In direct analogy with GBNs and Equation \eqref{eq:gbn}, we \rw{can} define a graphical model which is parametrized by a Cholesky factorization, up to a permutation. Let $\mathcal{G} = (V, E)$ be an arbitrary given acyclic digraph, and denote with $\prec$ its ancestral order. Let $\pr_{\prec}(i) = \{j \in V: j \prec i\}$ be the \emph{predecessor} set of a node $i \in V$.

Analogous to $\mathcal{U}(\mathcal{G})$ in GBNs, denote as $\mathcal{L}(\mathcal{G})$ the set of matrices which
have a zero pattern compatible with $\mathcal{G}$. By this we mean that if $(j, i) \not\in E$, \rw{$j \neq i$}, then $t_{ij} = 0$.
We may define the Gaussian graphical model (compar\rw{able to} Equation \eqref{eq:gbn}) 
\begin{equation}\label{eq:model} 
\mathcal{N}(\vecb{0}, \mat{\Sigma}) \text{ s.t. } \mat{\Sigma} = \mat{T}\mat{T}^t \text{ with }
\mat{T} \in \mathcal{L}(\mathcal{G}).  
\end{equation}

\begin{remark}\label{rem:model}
	As \rw{in the case of} GBNs (Remark \ref{rem:gbn:chol}), the parameter matrix $\mat{T}$ in Equation \eqref{eq:model} \rw{is} lower triangular, and thus coincide\rw{s} with the Cholesky factor of $\mat{\Sigma}$, when the variables are already ancestrally ordered. Otherwise, the sparse Cholesky factorization model applies when the ancestral order $\prec$ is known, that is, denoting as $\kappa(\mat{M})$ the permutation of rows and columns of $\mat{M}$ following $\prec$, then $\kappa(\mat{T})$ is the Cholesky factor of $\kappa(\mat{\Sigma})$.
\end{remark}

This extension has a more natural correspondence with the second regression interpretation of Remark \ref{rem:reg}, which holds after reordering rows and columns of $\mat{T}$ to \rw{comply with} $\prec$, as follows. First note that Equation \eqref{eq:l} holds for an unordered version of $\mat{L}$, and thus we have that $t_{ij} \propto \beta_{ij|\pr_{\prec}(j)}$ for $j \prec i$. Therefore, we retrieve a sort of \emph{ordered Markov property} (compar\rw{able to} Equation \eqref{eq:Uordered})
\begin{equation}
	X_i \indep X_j | \vecb{X}_{\pr_{\prec}(j)} \text{ for all } j\prec i \text{ and }j\not\in\pa(i).
\end{equation}

A simple example of an arbitrary graph would be \rw{that} in Figure \ref{fig:rbm3}.
In this model, factor $\mat{T}$ would be lower triangular after reordering its rows and columns following the ancestral ordering of the graph, $2 \prec 1 \prec 3$,
\begin{equation*}
	\mat{T} = \begin{pmatrix}
	* & * & 0\\
	0 & * & 0 \\
	* & 0 & *
	\end{pmatrix},\, 
	\kappa(\mat{T}) = \begin{pmatrix}
	* & 0 & 0\\
	* & * & 0\\
	0 & * & *
	\end{pmatrix}.
\end{equation*}
\begin{figure}[h]
	\centering
	\begin{tikzpicture}
	\node (x1) at (0,0) {$X_2$};
	\node (x2) at (2.5,0) {$X_1$};
	\node (x3) at (5,0) {$X_3$};
	\draw[->] (x1) -> (x2);
	\draw[->] (x2) -> (x3);
	\end{tikzpicture}
	\caption{Graph with ancestral order $2 \prec 1 \prec 3$. We have omitted the index set notation for vertices and directly used variables for clarity.}
	\label{fig:rbm3}
\end{figure}
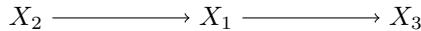

In the example of Figure \ref{fig:rbm1}, where variables already exhibit a natural order, the graph that would represent such interactions would be \rw{that} in Figure \ref{fig:rbm2}, whose parameter matrix $\mat{T}$ is already lower triangular (see also Remark \ref{rem:model}).
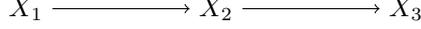
\begin{figure}[h]
	\centering
	\begin{tikzpicture}
	\node (x1) at (0,0) {$X_1$};
	\node (x2) at (2.5,0) {$X_2$};
	\node (x3) at (5,0) {$X_3$};
	\draw[->] (x1) -> (x2);
	\draw[->] (x2) -> (x3);
	\end{tikzpicture}
	\caption{Graph corresponding to the model in Figure~\ref{fig:rbm1}. We have omitted the index set notation for vertices and directly used variables for clarity.}
	\label{fig:rbm2}
\end{figure}

A further study of this graphical model extension is out of the scope of this work, since we focus on naturally ordered variables from hidden signal sources, but we have discussed it here for completeness and interest for future work.

\section{Model estimation}\label{sec:learn}
We will first review two regression-based existing estimators for this model that can be found in \cite{rothman2010}, and then will detail our proposed penalized matrix loss estimation method.

\subsection{Existing work: Banding and lasso}
Throughout this section denote as $\vecb{x}_i$ a sample of size $N$ corresponding to variable $X_i$, where $\vecb{X} = (X_1, \ldots, X_p)$ is assumed to follow the regression model of Equation \eqref{eq:reg:cov}.

The banding estimate for $\mat{T}$ builds upon the respective for $\mat{L}$. The idea is to estimate by standard least squares only the first $k$ sub-diagonals of $\mat{L}$ and set the rest to zero. Specifically, if $b(k) = \max(1, i - k)$ denotes the starting index, with respect to the band parameter $k$, of the $i$-th row vector $\vecb{l}_i = (l_{ib(k)}, \ldots, l_{ii-1})^t$ in matrix $\mat{L}$, then, letting $\hat{\vecb{\varepsilon}}_{b(k)} = \vecb{x}_{b(k)}$,
\begin{equation}\label{eq:optim:lsband}
\begin{aligned} 
&\hat{\vecb{l}}_i =\arg\min_{\vecb{l}_i}\, \lVert \vecb{x}_i - (\hat{\vecb{\varepsilon}}_{b(k)} \cdots \hat{\vecb{\varepsilon}}_{i - 1})\vecb{l}_i\rVert_2^2,\\
&\hat{\vecb{\varepsilon}}_i = \vecb{x}_i - (\hat{\vecb{\varepsilon}}_{b(k)} \cdots \hat{\vecb{\varepsilon}}_{i - 1})\hat{\vecb{l}}_i,\\
&\hat{d}_{ii} = \frac{1}{N}\lVert\hat{\vecb{\varepsilon}}_i\rVert_2^2,
\end{aligned} 
\end{equation}
where $\lVert \cdot \rVert_2$ is the Euclidean or $l_2$-norm. In order to ensure positive definiteness of all matrices involved in the computations, $k$ must be smaller than $\min(N - 1, p)$ \cite{rothman2010}. Matrix $\widehat{\mat{T}} = \widehat{\mat{L}}\sqrt{\widehat{\mat{D}}}$ inherits the band structure from $\widehat{\mat{L}}$. The main drawback of this banding estimator is the restrictive zero pattern that it imposes. \rw{Note also that this method requires previous selection of the parameter $k$.}

An alternative to banding which gives more flexibility over the zero pattern is to use lasso regression over Equation \eqref{eq:reg:cov},
\begin{equation}\label{eq:optim:lasso}
\begin{aligned}
&\hat{\vecb{l}}_i =\arg\min_{\vecb{l}_i}\, \lVert \vecb{x}_i - (\hat{\vecb{\varepsilon}}_{1} \cdots \hat{\vecb{\varepsilon}}_{i - 1})\vecb{l}_i\rVert_2^2 + \lambda\lVert\vecb{l}_i\rVert_1,\\
&\hat{\vecb{\varepsilon}}_i = \vecb{x}_i - (\hat{\vecb{\varepsilon}}_{1} \cdots \hat{\vecb{\varepsilon}}_{i - 1})\hat{\vecb{l}}_i,\\
&\hat{d}_{ii} = \frac{1}{N}\lVert\hat{\vecb{\varepsilon}}_i\rVert_2^2,
\end{aligned} 
\end{equation}
where this time $\hat{\vecb{\varepsilon}}_1 = \vecb{x}_1$, $\vecb{l}_i = (l_{i1}, \ldots, l_{ii-1})^t$, and $\lVert \cdot \rVert_1$ is the $l_1$-norm \rw{and $\lambda > 0$ is the penalisation parameter}. Observe that such penalty could be replaced with any other sparsity inducing penalty over $\vecb{l}_i$, for example, \cite{rothman2010} discuss\rw{ed} also the nested lasso \cite{levina2008} because the sparsity pattern is preserved in the covariance matrix. We have discarded such penalty because we are primarily interested in an arbitrary zero pattern in $\widehat{\mat{T}}$ and not worried \rw{about} the induced one in $\widehat{\mat{\Sigma}}$.

\subsection{Penalized gradient-based learning of the covariance sparse Cholesky factor}
The above approaches are based on the regression interpretation of the sparse Cholesky factor model, Equation \eqref{eq:reg:cov}. By contrast, we propose to directly estimate all of the parameters, that is, matrix $\mat{T}$ entries, by solving one optimization problem. This allows for example to recover maximum likelihood estimates, as well as have the potential to be easily extended to the graphical model interpretation (Equation \eqref{eq:model}) following an approach similar to \cite{zheng2018}.

Denote as $\mat{\Sigma}(\mat{T})$ the parametrization of a covariance matrix $\mat{\Sigma}$ with its Cholesky factor $\mat{T}$ (Equation \eqref{eq:chol:cov}).
We propose to learn a sparse model for $\mat{T}$ by solving the following optimization problem
\begin{equation}\label{eq:optim}
	\arg \min_{\mat{T}}  \phi\left( \mat{\Sigma}(\mat{T}) \right) + \lambda \lVert \mat{T} \rVert_1, 
\end{equation}
where $\phi(\cdot)$ is a differentiable loss function over covariance matrices, $\lambda \rw{> 0}$ is the penalisation parameter, and $\lVert \cdot \rVert_1$ is the $l_1$-norm \rw{for matrices}, which induces sparsity on $\mat{T}$ \cite{bach2012}. Note that, as in the regression case, the $l_1$ penalty could be replace\rw{d} with any other sparsity inducing matrix norm.
Solving Equation \eqref{eq:optim} can be done via proximal gradient algorithms, which have optimal convergence rates among first-order methods \cite{bach2012} and are tailored for a convex $\phi$ but also competitive in the non\rw{-}convex case \cite{varando2020}. In this work we have used two such smooth loss functions: the negative Gaussian log-likelihood and the triangular Frobenious norm. 

The negative Gaussian log-likelihood for a sample $\vecb{x}_1, \ldots, \vecb{x}_N$ when $\vecb{\mu}$ is assumed to be zero is \cite{anderson2003} proportional to
\begin{equation}\label{eq:nll}
	\phi_{NLL}(\mat{\Sigma}) = \ln\det(\mat{\Sigma}) + \tr(\mat{\Sigma}^{-1}\hat{\mat{\Sigma}}),
\end{equation}
where $\hat{\Sigma} = 1/N\sum_{n = 1}^N\vecb{x}_n\vecb{x}_n^t$ is the maximum likelihood estimator for $\Sigma$. On the other hand, the Frobenious norm loss that we also consider is
\begin{equation}\label{eq:frob}
	\phi_{FR}(\mat{\Sigma}) = \lVert \mat{\Sigma} - \hat{\mat{\Sigma}} \rVert_{F}^2 = \sum_{i = 1}^p\sum_{j = 1}^p(\sigma_{ij} - \hat{\sigma}_{ij})^2.
\end{equation}

Both $\phi_{NLL}$ and $\phi_{FR}$ are smooth, and in general $\phi_{NLL}$ renders the optimization problem of Equation \eqref{eq:optim} non-convex \cite{boyd2004}, whereas $\phi_{FR}$ is a convex function. \rw{I}n Appendix \ref{app:gradient} we provide the details on gradient computations for $\phi_{NLL}$ and $\phi_{FR}$, as well as the proximal algorithm pseudocode.

\section{Experiments}\label{sec:exp}

In all of the experiments we compare the four estimation methods outlined in the previous section: banding $\mat{T}$ (Equation \eqref{eq:optim:lsband}), lasso regressions (Equation \eqref{eq:optim:lasso}), and our two proposed penalized losses $\phi_{NLL}$ (Equation \eqref{eq:nll}) and $\phi_{FR}$ (Equation \eqref{eq:frob}). These four methods will be denoted in the remainder as $\mband$, $\mlasso$, $\mglik$ and $\mgfrob$, respectively. All data was standardized, and therefore in our proposed losses \rw{the sample correlation matrix} was used instead of $\widehat{\mat{\Sigma}}$. The implementation of our loss optimization methods $\mgfrob$ and $\rw{\mglik}$ can be found in the
R \cite{rcore} package \emph{covchol}\footnote{Version \rw{under} development:
{https://github.com/irenecrsn/covchol}. The experiments described throughout this section can be reproduced following the instructions and using the code available at the repository {https://github.com/irenecrsn/chol-inv}.}.

\subsection{Simulation}
We have conducted simulation experiments in two different scenarios. First, because as the work of Rothman et al. \cite{rothman2010} is the most directly related \rw{to} our model, we have replicated their simulation setting for completeness. Therein they select three fixed covariance matrices with either a fixed known banded sparsity pattern or no zeros at all. By contrast, in the second experiment we explore how the methods perform when the sparsity pattern is arbitrary.

In both experiments we have measured two statistics in order to asses\rw{s} both model selection (how precise the zero pattern is recovered) and estimation (how numerically close is the retrieved matrix to the original one). These metrics are evaluated over $\mat{\Sigma}$ and $\widehat{\mat{\Sigma}}$ in the second experiment instead of $\mat{T}$ and $\widehat{\mat{T}}$, for better comparability with \cite{rothman2010}. Specifically, we use the $\fone$ score for evaluating the zero pattern,
\begin{equation}\label{eq:f1}
\fone(\mat{T}, \widehat{\mat{T}}) = 2\frac{\tpr(\mat{T}, \widehat{\mat{T}})\tdr(\mat{T}, \widehat{\mat{T}})}{\tpr(\mat{T}, \widehat{\mat{T}}) + \tdr(\mat{T}, \widehat{\mat{T}})},
\end{equation}
where $\tpr$ and $\tdr$ are the true positive and discovery rate, respectively,
\begin{align*}
&\tpr(\mat{T}, \widehat{\mat{T}}) = \frac{|\{t_{ij} \neq 0 \text{ and } \hat{t}_{ij} \neq 0\}|}{|\{t_{ij} \neq 0\}|},\\
&\tdr(\mat{T}, \widehat{\mat{T}}) = \frac{|\{t_{ij} \neq 0 \text{ and } \hat{t}_{ij} \neq 0\}|}{|\{\hat{t}_{ij} \neq 0\}|};
\end{align*}
and the induced matrix $1$-norm,
\begin{equation*}
\norm(\mat{T}, \widehat{\mat{T}}) = \lVert \mat{T} - \hat{\mat{T}} \rVert_1 = \max_{1 \leq j \leq p}\sum_{i = 1}^p \lvert \hat{t}_{ij} - t_{ij} \rvert.
\end{equation*}

\subsubsection{Fixed covariance matrices}
The fixed covariance matrices used in the simulations by Rothman et al. \cite{rothman2010} are:
\begin{itemize}
	\item The autoregressive model of order $1$, where the true covariance
	matrix $\mat{\Sigma}_1$ has entries $\sigma_{ij} = \rho^{|i - j|}$, with $\rho = 0.7$.
	\item The $4$-banded correlation matrix $\mat{\Sigma}_2$ with entries
	$\sigma_{ij} = 0.4\mathbb{I}(|i-j| = 1) + 0.2\mathbb{I}(2\leq|i-j|\leq 3) +
	0.1\mathbb{I}(|i-j| =4)$ for $i \neq j$, $\mathbb{I}$ being the set indicator function.
	\item The dense correlation matrix $\mat{\Sigma}_3$ with $0.5$ in all of
	its entries except for the diagonal.
\end{itemize}
Similarly to \cite{rothman2010}, we use for matrix dimension $p$ values ranging from $30$ to $500$, and draw from
the respective distribution $\mathcal{N}(\vecb{0}, \mat{\Sigma}_i)$, $i \in \{1, 2, 3\}$, a sample of size $200$ which allows to visualize both the $p > N$ and $p < N$ scenarios. With this experiment we measure how sparsity inducing methods for learning $\mat{T}$ behave in scenarios which are not specially suited for them, except for $\mband$ and $\mat{\Sigma}_2$.

Figure \ref{fig:sim:fixedsigmas} shows the results of the
aforementioned simulation scenario. The norms are shown in logarithmic scale for a better comparison between the methods, since there were significant disparities.
\begin{figure}[h]
	\centering
	\includegraphics[width = \linewidth]{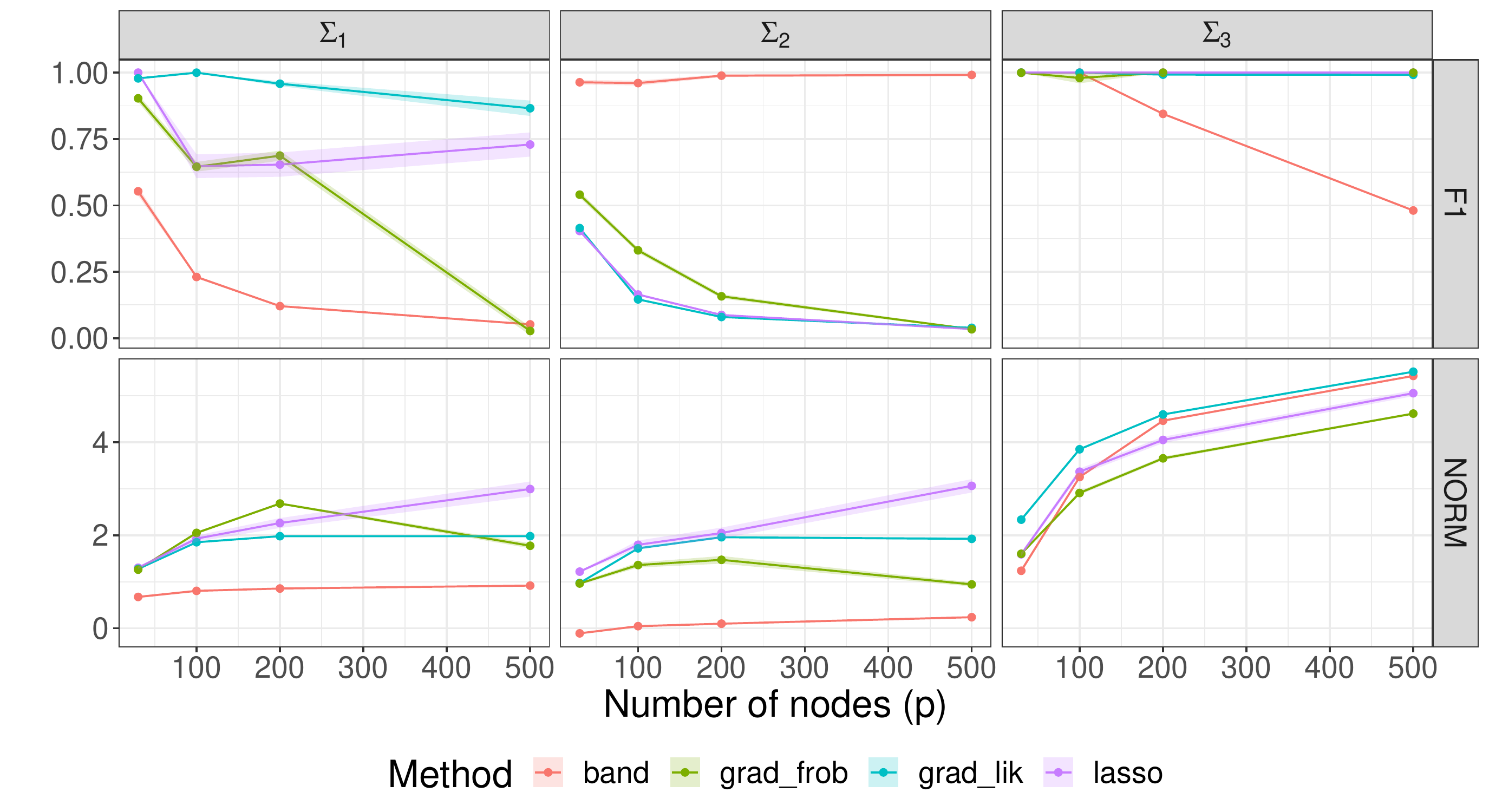}
	\caption{Results of the simulation experiment \rw{set out in} \cite{rothman2010}. Metric $\norm$ is in logarithmic scale.}
	\label{fig:sim:fixedsigmas}
\end{figure}
$\mat{\Sigma}_1$ and $\mat{\Sigma}_3$ are both dense matrices, with $\mat{\Sigma}_1$ having entries that decay when moving away from the diagonal. We observe that the inexistent sparsity pattern is best approximated by $\mglik$ and $\mlasso$, but interestingly $\mgfrob$ and $\mband$ achieve competitive norm results, sometimes even outperforming the rest. Matrix $\mat{\Sigma}_2$ is banded, therefore\rw{,} as expected\rw{,} $\mband$ achieves both the highest $\fone$ measure and lowest norm difference. 

\subsubsection{Arbitrary sparsity pattern in the Cholesky factor $\mat{T}$}

In this experiment the sparse Cholesky factor $\mat{T}$ is simulated using essentially the method of \cite[Algorithm 3]{cordoba2020} with a random acyclic digraph to represent zero pattern, that is, the latent structure (see Figure \ref{fig:rbm1}). Observe that in general this does not yield a uniformly sampled Cholesky factor, but it is more flexible that the standard diagonal dominance procedure, see \cite{cordoba2020} for further discussion on this issue. We \rw{generate} three Cholesky factors $\mat{T}_i$ with a proportion of $i/p$ non-zero entries (density/average edge number of the corresponding acyclic digraph), where $i \in \{1, 2, 3\}$. Then we draw a sample of size $200$ from $\mathcal{N}(\vecb{0}, \mat{T}_i\mat{T}_i^t)$ for $p$ ranging between $30$ and $500$, as in the previous experiment. 

Figure \ref{fig:sim:randl} depicts the results.
\begin{figure}[h]
	\centering
	\includegraphics[width = \linewidth]{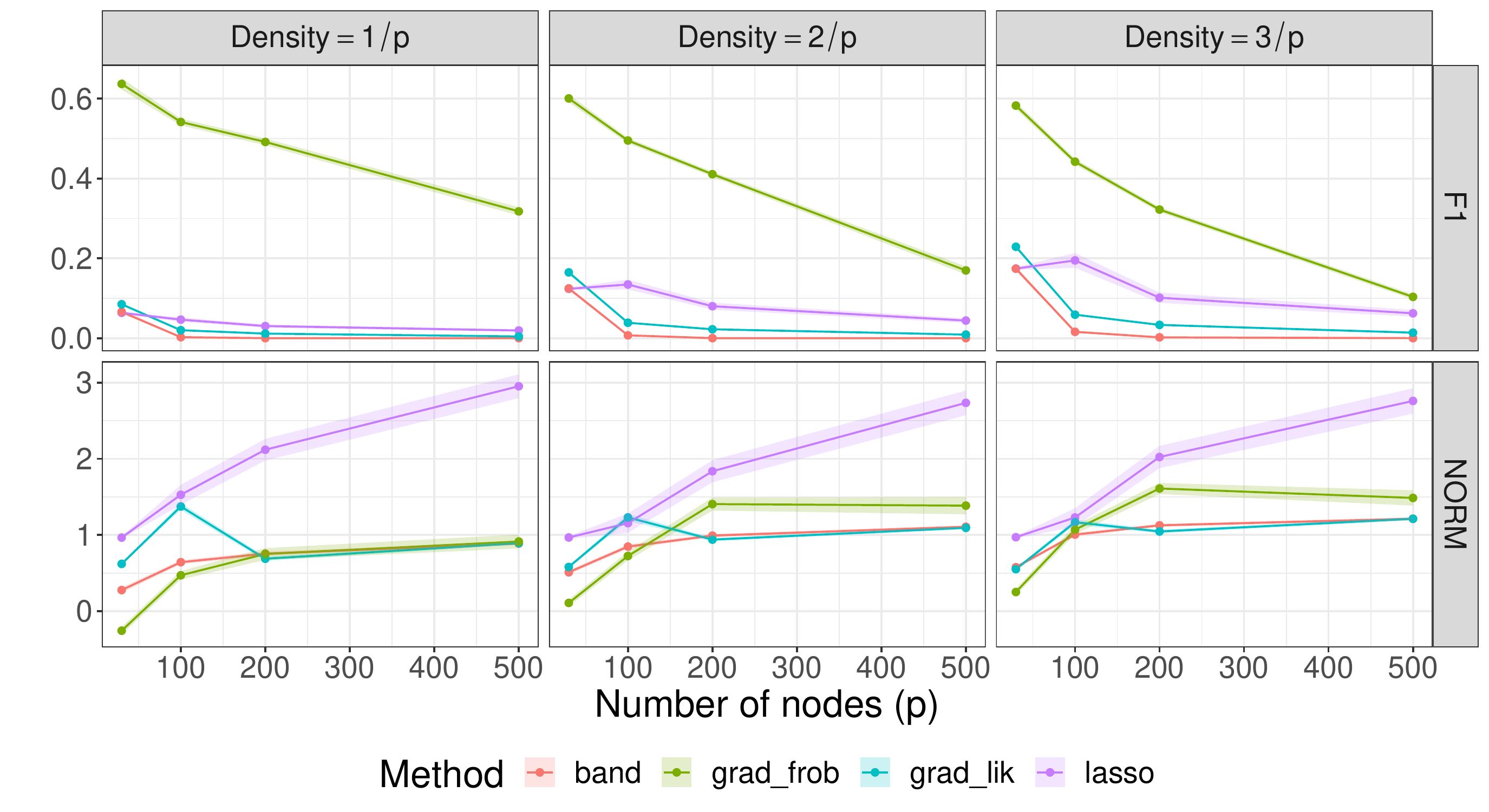}
	\caption{Results of the simulation experiment for an arbitrary sparsity 	pattern in $\mat{T}$. Metric $\norm$ is in logarithmic scale. \rw{Density indicates the average proportion of lower triangular non-zero entries in the simulated $\mat{T}$ Cholesky factors.}}
	\label{fig:sim:randl}
\end{figure}
Note that as the density decreases, the $\fone$ score and matrix norm results slightly worsen, but in general the methods behaviour is maintained. The $\mband$ estimator exhibits a performance similar to the previous experiment: although achieving a small $\fone$ score, it has a relatively small matrix norm difference. This behaviour is shared in this case with $\mglik$, which has in general poor performance. However, the worst performing method is $\mlasso$, which neither is able to recover the sparsity pattern, nor gets numerically close to the original Cholesky factor (it has a significantly high value for the norm difference). Conversely, method $\mgfrob$ has the best performance, with a significantly high $\fone$ score when compared with the rest and competitive or best norm difference results.

\subsection{Real data}

In this section we have selected two data sets from the UCI machine learning repository \cite{dua2020} where a natural order arises among the variables, and which are therefore suitable for our sparse covariance Cholesky factorization model. Both of them are labelled with a class variable, therefore after estimating the respective Cholesky factors with each method, we assess classification performance.

For classifying a sample $\vecb{x}$ we use quadratic discriminant analysis \cite{rothman2010}, where $\vecb{x}$ is classified in the class value $c$ that maximizes
\begin{equation*}
\begin{aligned}
\ln\hat{f}(\vecb{x}, c) 
&= \ln\hat{f}(\vecb{x} | c) + \ln\hat{f}(c)\rw{,}
\end{aligned}
\end{equation*}
where $\hat{f}(c)$ is the proportion of observations for class $c$ and we have expressed $\hat{f}(\vecb{x} | c)$ in terms of $\mat{T}$ instead of $\mat{\Sigma}$,
\begin{equation}\label{eq:nllclass}
\begin{aligned}
\ln\hat{f}(\vecb{x} | c) 
&\propto \frac{1}{2}\ln\det(\widehat{\mat{\Sigma}}_c) - \frac{1}{2}(\vecb{x} - \hat{\vecb{\mu}}_c)^t \widehat{\mat{\Sigma}}_c^{-1} (\vecb{x} - \hat{\vecb{\mu}}_c)\\
&= \ln\det(\widehat{\mat{T}}_c) - \frac{1}{2}(\vecb{x} - \hat{\vecb{\mu}}_c)^t \widehat{\mat{T}}_c^{-t}\widehat{\mat{T}}_c^{-1} (\vecb{x} - \hat{\vecb{\mu}}_c)\\
&= \sum_{i = 1}^p\ln t_{ii} - \frac{1}{2}(\widehat{\mat{T}}_c^{-1}(\vecb{x} - \hat{\vecb{\mu}}_c))^t \widehat{\mat{T}}_c^{-1}(\vecb{x} - \hat{\vecb{\mu}}_c),
\end{aligned}
\end{equation}
with $\hat{\vecb{\mu}}_c$, $\widehat{\mat{\Sigma}}_c$ and $\widehat{\mat{T}}_c$ the respective estimates from training samples belonging to class $c$.

Finally, for evaluating classification performance we have used:
\begin{itemize}
	\item The $\fone$ score, already defined in Equation \eqref{eq:f1} in terms of $\tdr$ and $\tpr$, but adapted to classification instead of matrix entries.
	\item The true negative rate, $\tnr$, since it is not contained \rw{in} the $\fone$ score, which is the proportion of observations that have been correctly \emph{not} classified as class $c$.
	\item The accuracy, $\acc$, which measures the proportion of observations that have been correctly assigned a class. Observe that this last metric, unlike the other two, is not class-dependent, but instead global.
\end{itemize}

\subsubsection{Sonar: Mine vs. Rocks}

The first real data set we explore is the \emph{Connectionist Bench (Sonar, Mines vs. Rocks)} data set, which contains numeric observations from a sonar signal bounced at both a metal cylinder (mine) and rocks. It contains $60$ variables and $208$ observations. Each variable corresponds to the energy within a certain frequency band, integrated over a period of time, in increasing order. Each observation represents a different beam angle for the same object under detection. Over this data set the objective is to classify a sample as rock or mine. This data set was also analysed in \cite{rothman2010}, but without the expression in terms of $\mat{T}$ for Equation \eqref{eq:nllclass} and only using method $\mband$ for $\mat{T}$.

As a first exploratory step, we have applied each of the methods for learning the Cholesky factor $\mat{T}$ to all the instances labelled as $M$ (mines), and $R$ (rocks), which we show as a heatmap in Figure \ref{fig:sonar:chols}.
\begin{figure}[h]
	\centering
	\includegraphics[width = \linewidth]{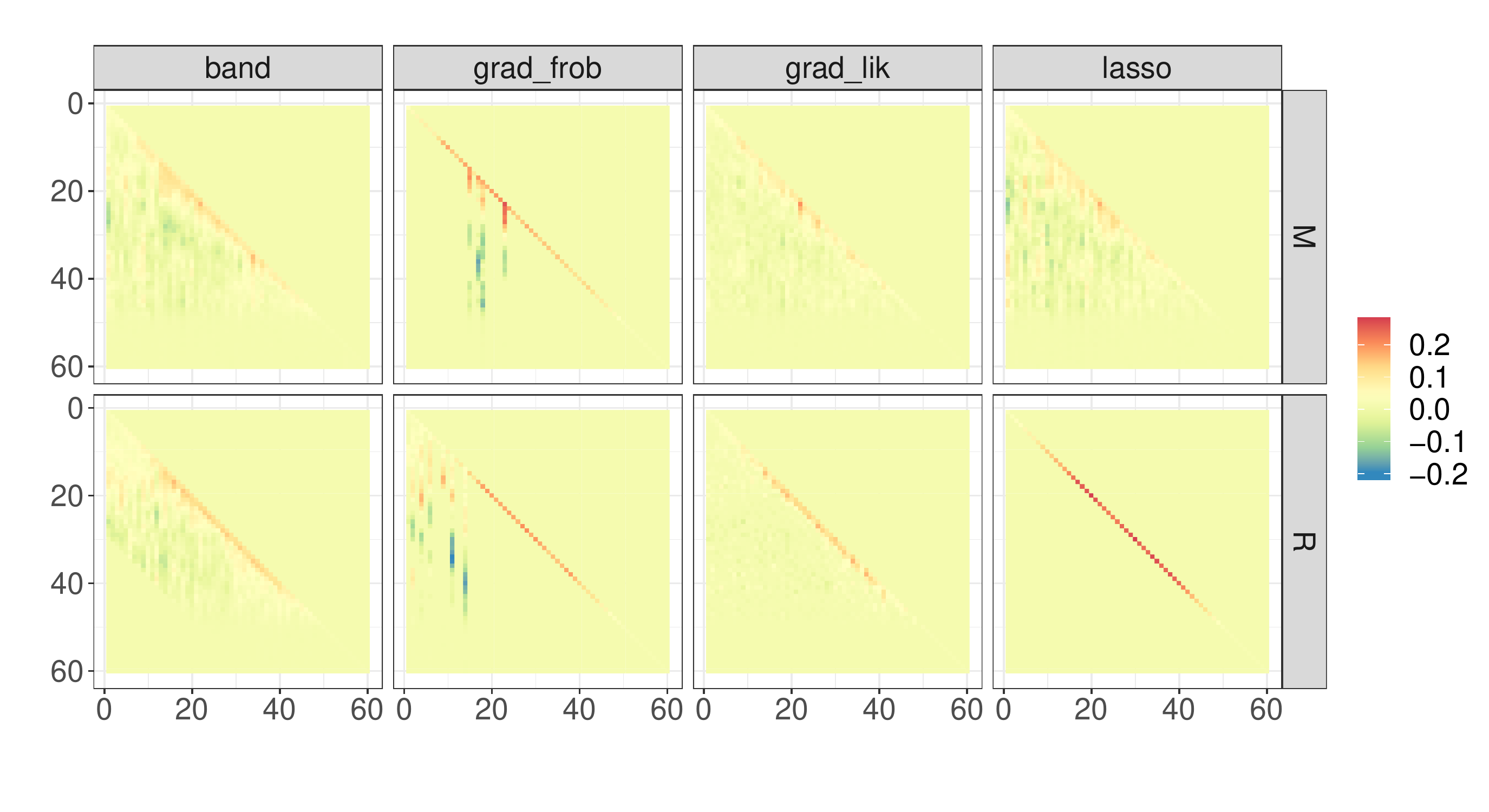}
	\caption{Heatmaps of the Cholesky factors of rock and mine samples. M: Mines; R: Rocks.}
	\label{fig:sonar:chols}
\end{figure}
The Cholesky factor for mines retrieved by $\mglik$ and $\mlasso$ look fairly similar, whereas the one for rocks that $\mlasso$ estimates is nearly diagonal. Bands can be clearly observed from heatmaps by $\mband$, and all methods impose zero values for variables near to or higher than $50$, which could be motivated by the problem characteristics and hint at high sonar frequencies being nearly noiseless. This latter behaviour is inherited by the covariance matrices, whose heatmaps are shown in Appendix \ref{app:fig}, Figure \ref{fig:sonar:covs}. The entries in the Cholesky factor estimated by $\mgfrob$ are the most extreme, since most of them are zero, and the ones which are not have the highest and lowest values among all the estimates recovered. Despite the outlined differences among the Cholesky factors, the induced covariance matrices exhibit rather similar heatmaps and eigenvalues (Figures \ref{fig:sonar:covs} and \ref{fig:sonar:scree} in Appendix \ref{app:fig}).

For the quadratic discriminant analysis we have used leave-one-out cross-validation, since the sample size was sufficiently small to allow it. Table \ref{tab:sonar} contains the results thus obtained. We see that $\mlasso$ is the method that performs poorest overall. Conversely, $\mband$ is arguably the best for this problem, except for the $\tnr$ of rock samples, which is highest for $\mgfrob$. However, observing the rest of statistics for $\mgfrob$, it can be deduced that this method over-classifies samples as mines: it has the lowest $\tnr$ for them. On the other hand, $\mglik$ performs competitively for this problem, but is in general outperformed or matched by $\mband$. Since the sonar behaviour hints at a band structure for the covariance (frequency patterns being related \rw{to} those close to them), and therefore for its Cholesky factor, the good performance of $\mband$ could be expected.
\begin{table}[h]
	\centering
	\begin{tabular}{lllll}
		& $\mband$ & $\mgfrob$ & $\mglik$ & $\mlasso$\\
		$\tnr$ (M) & \textbf{0.78} & 0.08 & \textbf{0.78} & 0.62\\
		$\fone$ (M) & \textbf{0.8} & 0.7 & 0.55 & 0.33\\
		$\tnr$ (R) & 0.79 & \textbf{0.97} & 0.45 & 0.26\\
		$\fone$ (R) & \textbf{0.78} & 0.15 & 0.65 & 0.5\\
		$\acc$ & \textbf{0.79} & 0.56 & 0.61 & 0.43
	\end{tabular}
	\caption{Statistics for the sonar problem. M: mines; R: rocks.}
	\label{tab:sonar}
\end{table}

\subsubsection{Wall-\rw{F}ollowing \rw{R}obot \rw{Navigation}}
The other real data set we use is the \emph{Wall-Following Robot Navigation} one. Here a robot moves in a room following the wall clockwise. It contains $5456$ observations and $24$ variables. Each variable corresponds to the value of an ultrasound sensor, which are arranged circularly over the robot's body. Here the increasing order reflects the reference angle where the sensor is located. Since the robot is moving clockwise, here the classification task is between four possible class values: Move\rw{-}Forward, Sharp\rw{-}Right\rw{-}Turn, Slight\rw{-}Left\rw{-}Turn or Slight\rw{-}Right\rw{-}Turn.

As in the previous problem, we obtain the Cholesky factors for each of the movements, depicted in Figure \ref{fig:robot:chols}.
We notice that $\mgfrob$ outputs a similar matrix (except for the Slight\rw{-}Left\rw{-}Turn movement) than the other three methods, which means that the extreme behaviour we observed in the sonar experiment was problem related. By contrast, the Cholesky factor for Slight\rw{-}Left\rw{-}Turn is nearly diagonal. The other matrices are rather similar among the methods, with $\mband$ notably choosing in general a high banding parameter $k$ (few to no bands). Here we have a similar structure as in the sonar problem: we appreciate for all the movements except Slight\rw{-}Left\rw{-}Turn that most entries close to the diagonal are positive, whereas distant ones are frequently negative. Regarding Slight\rw{-}Left\rw{-}Turn, these matrices are the sparsest and have near zero values on the diagonal. Since the robot is moving clockwise, this movement is related to obstacles, therefore it could hint that sensor readings are correctly identifying them.
\begin{figure}[h]
	\centering
	\includegraphics[width = \linewidth]{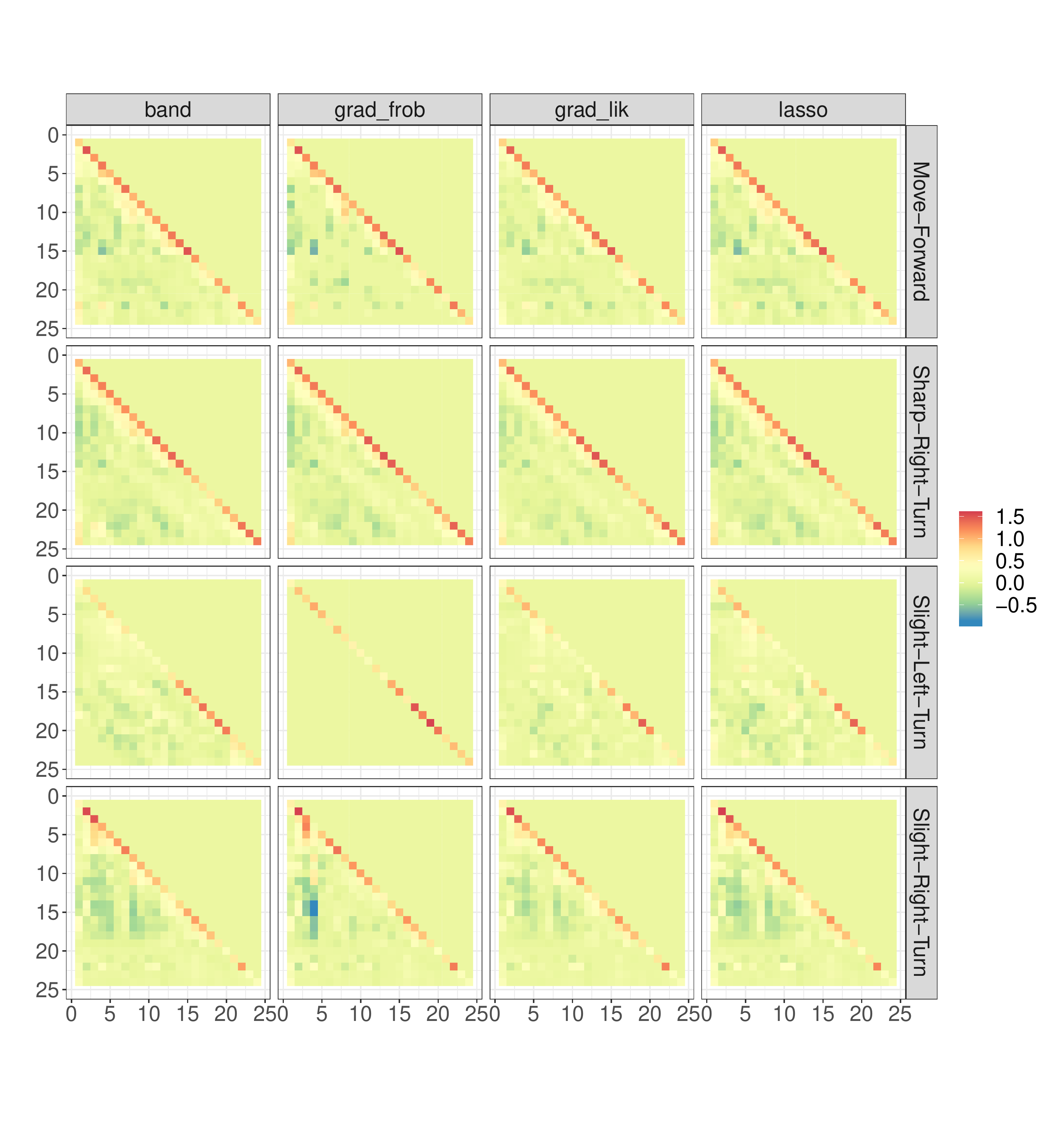}
	\caption{Heatmaps of the Cholesky factor of the wall-following robot \rw{navigation} samples.}
	\label{fig:robot:chols}
\end{figure}

In this problem we have a larger sample size, and therefore we split the data into train and test, with half of the samples on each set. The classification results are shown in Table \ref{tab:robot}. We observe that all of the methods perform arguably good, in fact they achieve nearly identical accuracy. It is noticeable how competitive are $\mlasso$ and $\mglik$, which performed much worse in the sonar problem. We also notice that arguably the best results are obtained for the Slight\rw{-}Left\rw{-}Turn movement, which confirms our previous intuition over the heatmaps about sensors correctly identifying obstacles. The worst performance over all methods is for the Slight\rw{-}Right\rw{-}Turn movement, but is not \rw{noteworthy} when compared with the rest (except for Slight\rw{-}Left\rw{-}Turn).
\begin{table}[h]
	\centering
	\begin{tabular}{lllll}
		& $\mband$ & $\mgfrob$ & $\mglik$ & $\mlasso$\\
		$\tnr$ (MF) & \textbf{0.87} & 0.86 & 0.85 & 0.84\\
		$\fone$ (MF) & \textbf{0.64} & \textbf{0.64} & 0.61 & 0.62\\
		$\tnr$ (SHR) & \textbf{0.89} & 0.88 & 0.87 & 0.85\\
		$\fone$ (SHR) & 0.72 & 0.72 & \textbf{0.73} & \textbf{0.73}\\
		$\tnr$ (SLL) & 0.96 & 0.95 & \textbf{0.98} & \textbf{0.98}\\
		$\fone$ (SLL) & 0.67 & 0.65 & \textbf{0.72} & 0.7\\
		$\tnr$ (SLR) & 0.82 & 0.83 & 0.81 & \textbf{0.84}\\
		$\fone$ (SLR) & 0.56 & \textbf{0.57} & 0.55 & \textbf{0.57}\\
		$\acc$ & \textbf{0.66} & \textbf{0.66} & 0.65 & \textbf{0.66}
	\end{tabular}
	\caption{Statistics for the robot problem. MF: Move\rw{-}Forward; SHR: Sharp\rw{-}Right\rw{-}Turn; SLL: Slight\rw{-}Left\rw{-}Turn; SLR: Slight\rw{-}Right\rw{-}Turn.}
	\label{tab:robot}
\end{table}

\subsection{Discussion}
We can draw several conclusions from both the simulated and real experiments. \rw{First}, we report in Figure \ref{fig:time} the execution time for each method, where it can be seen that $\mglik$ is the slowest one and $\mband$ is the fastest.
\begin{figure}[h]
	\centering
	\includegraphics[width = \linewidth]{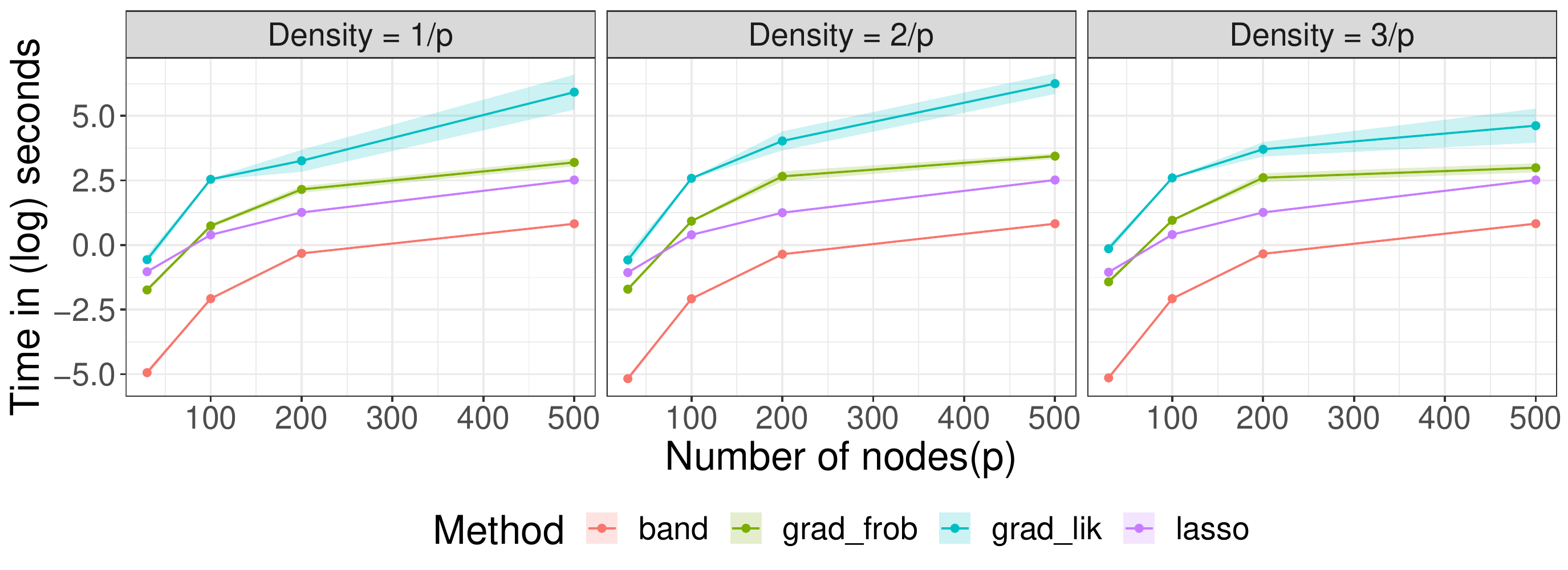}
	\caption{Logarithm of the execution time (in seconds) for each of the methods under evaluation. \rw{Density indicates the average proportion of lower triangular non-zero entries in the simulated $\mat{T}$ Cholesky factors.}}\label{fig:time}
\end{figure}

Whenever there is a clear dependence between variables that are close in the ordering, such as in the sonar example, the $\mband$ method could be preferred, because it is the one that more naturally approximates the structure induced in the Cholesky factor (as happened in the sonar example).

Our new proposed method $\mgfrob$ has shown to be competitive both in execution time as well as recovery results: when interested in model selection, that is, how accurately zeros in the Cholesky factor are estimated, it yields the best results. Conversely, $\mglik$ has shown to be the most robust: in simulations it achieved reasonable performance even when the true covariance matrix was dense, and it also performed competitively in the sonar example, which was mostly suited for $\mband$ as we discussed.

Finally, $\mlasso$ has achieved overall poor results, except for the wall-following \rw{robot navigation} data set. Specially, in simulations it failed to correctly recover the zero pattern in the Cholesky factor and was the numerically farthest away from the true matrix. Despite this, it is the second fastest of the four methods, so when model selection or robustness \rw{are not a concern} it is a good alternative to $\mband$, since it provides more flexibility over the zero pattern in the Cholesky factor.

\section{Conclusions and future work}\label{sec:conc}
In this paper we have proposed a sparse model for the Cholesky factor of the covariance matrix of a multivariate Gaussian distribution. Few other works in the literature have previously addressed this problem, and we expand them in many ways. We have formalise\rw{d} the extension of an arbitrary zero pattern in the Cholesky factorization to a Gaussian graphical model. We have proposed a novel estimation method that directly optimizes a penalized matrix loss, and we have compared it with the other already existing regression-based approaches in different simulation and real experiments. We have finally discussed which estimation method is preferable depending on the problem characteristics, and argue how our estimation proposal is preferable in many scenarios than those already existing.

As future research, the most direct and interesting derivative work would be to further analyse, both theoretically and empirically, the Gaussian graphical model extension to unordered variables of the sparse covariance Cholesky factorization model. Also in this direction, its relationship with the already established covariance graph \cite{kauermann1996,cox1993,chaudhuri2007,bien2011}, an undirected graphical model over the covariance matrix, could yield interesting results. Regarding our novel estimation method, we could explore other alternatives to solve the optimization problems that arise from the two losses that we compute in Appendix \ref{app:gradient}. For large data sets, the use of deep learning models to find the hidden Cholesky factor could be also explored.

\section*{Acknowledgements}
	This work has been partially supported by the Spanish \rw{Ministerio de Ciencia, Innovación y Universidades} through the
	PID2019-109247GB-I00 project. Irene Córdoba has been supported by the
	predoctoral grant FPU15/03797 from the Spanish \rw{Ministerio de Ciencia, Innovación y Universidades}. Gherardo Varando has been supported by a research
	grant (13358) from VILLUM FONDEN. We thank Adam P. Rothman for helpful discussion and code sharing about his work.

\appendix
\section{Appendices}

\subsection{A transformation between the Cholesky factorizations of $\mat{\Sigma}$ and $\mat{\Omega}$}\label{app:gbn:rel}
The following result gives how to obtain the elements of $\mat{L}$ from those in $\mat{U}$. Recall that, since $\mat{D}$ is diagonal, the $(i, i)$ entry in $\mat{D}^{-1}$ is $d_{ii}^{-1}$.
\begin{proposition*}
	For each $i \in
	\{1, \ldots, p\}$ with $j \in \{1, \ldots, i - 1\}$ the following
	identity holds: \[ \beta_{ij | 1, \ldots, j} = \beta_{ij | 1, \ldots, i
		- 1} + \sum_{k = j + 1}^{i - 1} \beta_{ik | 1, \ldots, i - 1} \beta_{kj
		| 1, \ldots, j}.  \] 
\end{proposition*}
\begin{proof}
	First we recall that $\mat{L} = \mat{U}^{-t}$ and
	$\mat{U}\mat{U}^{-1} = {I}_p$. Thus, for each $i
	\in \{2, \ldots, p\}$ and each $j < i$, if we multiply the $i$-th
	row in
	$\mat{U}$ with the $j$-th column of $\mat{U}^{-1}$, this must be
	equal to
	$0$. Specifically, we obtain the following equation,
	\begin{align*}
	0 = &- \beta_{i\rw{j} | 1, \ldots, i - 1} \\
	&- \beta_{i\rw{j + 1} | 1, \ldots, i - 1}\beta_{\rw{j + 1}j | 1, \ldots, j} \\
	& \vdots\\
	&- \beta_{i\rw{i} - 1  | 1, \ldots, i - 1}\beta_{\rw{i}-1j | 1, \ldots, j}\\
	&+ \beta_{ij | 1, \ldots, j}.
	\end{align*}
	Therefore,
	\begin{equation*}
	0 = \beta_{ij | 1, \ldots, j} - \sum_{k = j + 1}^{i - 1}\beta_{ik | 1, \ldots, i - 1}
	\beta_{kj | 1, \ldots, j} - \beta_{ij | 1, \ldots, i - 1},
	\end{equation*}
	which yields the desired result.
\end{proof}

\subsection{Proximal gradient algorithm and gradient computations for $\phi_{NLL}$ and $\phi_{FR}$}\label{app:gradient}

We will show how to obtain a simplified expression for the gradient of $\phi(\mat{\Sigma}(\mat{T}))$ with respect to $\mat{T}$ as a function of the one with respect to $\mat{\Sigma}$. We will denote these gradients as $\nabla_{\mat{T}}\phi$ and $\nabla_{\mat{\Sigma}}\phi$, respectively.
\begin{proposition*}
	For any differentiable loss function $\phi(\mat{\Sigma}(\mat{T}))$,
	\begin{equation}\label{eq:gradient}
	\nabla_{\mat{T}}\phi = 2\nabla_{\mat{\Sigma}}\phi\mat{T}.
	\end{equation}
\end{proposition*}
\begin{proof}
	This proof follows some ideas from \cite[Proposition 2.1]{varando2020}. 
	By matrix calculus \cite{petersen2008matrix}, we have the following gradient expression for $j < i$,
	\begin{equation}\label{eq:gradient:brute}
	\frac{\partial \phi(\mat{\Sigma}(\mat{T}))}{\partial t_{ij}} =
	tr\left(\nabla_{\mat{\Sigma}}\phi
	\frac{\partial\mat{\Sigma}(\mat{T})}{\partial t_{ij}}\right),
	\end{equation}
	where we have used that $\nabla_{\mat{\Sigma}}\phi$ is symmetric. Furthermore, note that\cite{petersen2008matrix}
	\[
	\frac{\partial \mat{\Sigma}(\mat{T}) }{\partial t_{ij}} = \frac{\partial
		\mat{T}\mat{T}^t }{\partial t_{ij}} = \mat{T}\mat{E}^{ij} +
	\mat{E}^{ji}\mat{T}^t\rw{,}
	\]
	where \rw{$\mat{E}^{ij}$ ($\mat{E}^{ji}$) has its $(i, j)$ ($(j, i)$) entry equal to one and zero elsewhere}. Then from Equation \eqref{eq:gradient:brute} we have
	\begin{align*}
	\frac{\partial\phi(\mat{\Sigma}(\mat{T}))}{\partial t_{ij}} =&
	\tr\left(\nabla_{\mat{\Sigma}}\phi(\mat{T}\mat{E}^{ij} +
	\mat{E}^{ji}\mat{T}^t)\right) \\
	= & \tr\left(\nabla_{\mat{\Sigma}}\phi \mat{T}\mat{E}^{ij} \right)
	+ \tr\left(\nabla_{\mat{\Sigma}}\phi \mat{E}^{ji}\mat{T}^t \right) \\
	= & \tr\left(\mat{E}^{ji}\mat{T}^t\nabla_{\mat{\Sigma}}\phi\right) +
	\tr\left(\mat{E}^{ji}\mat{T}^t\nabla_{\mat{\Sigma}}\phi\right)\\
	= & 2\tr\left(\mat{E}^{ji}\mat{T}^t\nabla_{\mat{\Sigma}}\phi\right).
	\end{align*}
	Since $a_{ji} = \tr(\mat{E}^{ji}\mat{A})$ \rw{for any matrix $\mat{A}$}, we have the desired result.
\end{proof}
The above proposition implies that once a loss function $\phi(\mat{\Sigma}(\mat{T}))$ is fixed, it is only necessary to compute $\nabla_{\mat{\Sigma}}\phi$ in order to obtain $\nabla_{\mat{T}}\phi$. We would then use it for the proximal gradient method (Algorithm \ref{alg:proxgradb}).

\begin{algorithm}[h]
	\caption{Proximal gradient algorithm for minimization of $l_1$-penalized loss}\label{alg:proxgradb}
	\begin{algorithmic}[1]
		\small
		\REQUIRE  $\phi$ differentiable \rw{function over positive definite matrices}, \\
		\hspace{8pt} $\mat{T} \in \mathbb{R}^{p\times p} $ 
		lower triangular, \\  
		\hspace{8pt} $M \in \mathbb{N}$,  
		 $\varepsilon, \lambda,  \alpha \in (0,1)$ 
		\STATE{$f = \phi(\mat{\Sigma}(\mat{T})) $}
		\STATE{$g =  \lambda ||\mat{T}||_1 $}
	        \REPEAT 
		\STATE{$\mat{D} = \nabla_{\mat{T}} 
		\phi(\mat{\Sigma(\mat{T}}))$}
                \STATE{$s = 1$} 
		\LOOP
		\STATE{$\mat{T}' = \mat{T} - s  {\mat{D}}$ }
		\STATE{soft thresholding $\mat{T}'$ at level $s\lambda$} 
		\STATE{$f' = \phi(\mat{\Sigma}(\mat{T}))$, 
		$g' =  \lambda ||\mat{T}||_1 $}
		\STATE{$\nu =\frac{1}{2s}(||\mat{T} - \mat{T}'||_F^2) 
		          + \tr( (\mat{T}' - \mat{T}){\mat{D}})$}
                \IF{$f' + g' \leq f + g$ \AND   $f' \leq f + \nu $ }
		\STATE{\textbf{break}}
		\ELSE 
		\STATE{$s = \alpha s$}
		\ENDIF
		\ENDLOOP 
		\STATE{$\delta = (f + g - f' - g') $}
		\STATE{$\mat{T} = \mat{T}', f =f', g = g'$}
		\UNTIL{$k > M$ \OR $\delta < \varepsilon$}
		\ENSURE   $T$ 
	\end{algorithmic}
\end{algorithm}

We thus can easily obtain the gradient for $\phi_{NLL}$ and $\phi_{FR}$, which we consider in this work. Standard matrix calculus
\cite{petersen2008matrix} gives $\nabla_{\mat{\Sigma}}\phi_{NLL} =
\mat{\Sigma}^{-1}-\mat{\Sigma}^{-1} \hat{\mat{\Sigma}} \mat{\Sigma}^{-1}$. Therefore we have
\begin{align*}
\nabla_{\mat{T}} \phi_{NLL} &= 2\nabla_{\mat{\Sigma}}
\phi_{NLL}\mat{T}\\
&= 2\mat{\Sigma}^{-1}(\mat{T})(\mat{I}_p - \hat{\mat{\Sigma}}\mat{\Sigma}^{-1}(\mat{T}))\mat{T}\\
&= 2\mat{T}^{-t}(\mat{I}_p - \mat{T}^{-1}\hat{\mat{\Sigma}}\mat{T}^{-t}).
\end{align*}
Conversely, from \cite{petersen2008matrix} we have that $\nabla_{\mat{\Sigma}} \phi_{FR}(\mat{\Sigma}) = 2(\mat{\Sigma} - \hat{\mat{\Sigma}})$. Thus $\nabla_{\mat{T}} \phi_{FR} = 2(\mat{T}\mat{T}^t - \hat{\mat{\Sigma}})\mat{T}$.

\subsection{More figures for real data experiments}\label{app:fig}

\begin{figure}[h!]
	\centering
	\includegraphics[width = \linewidth]{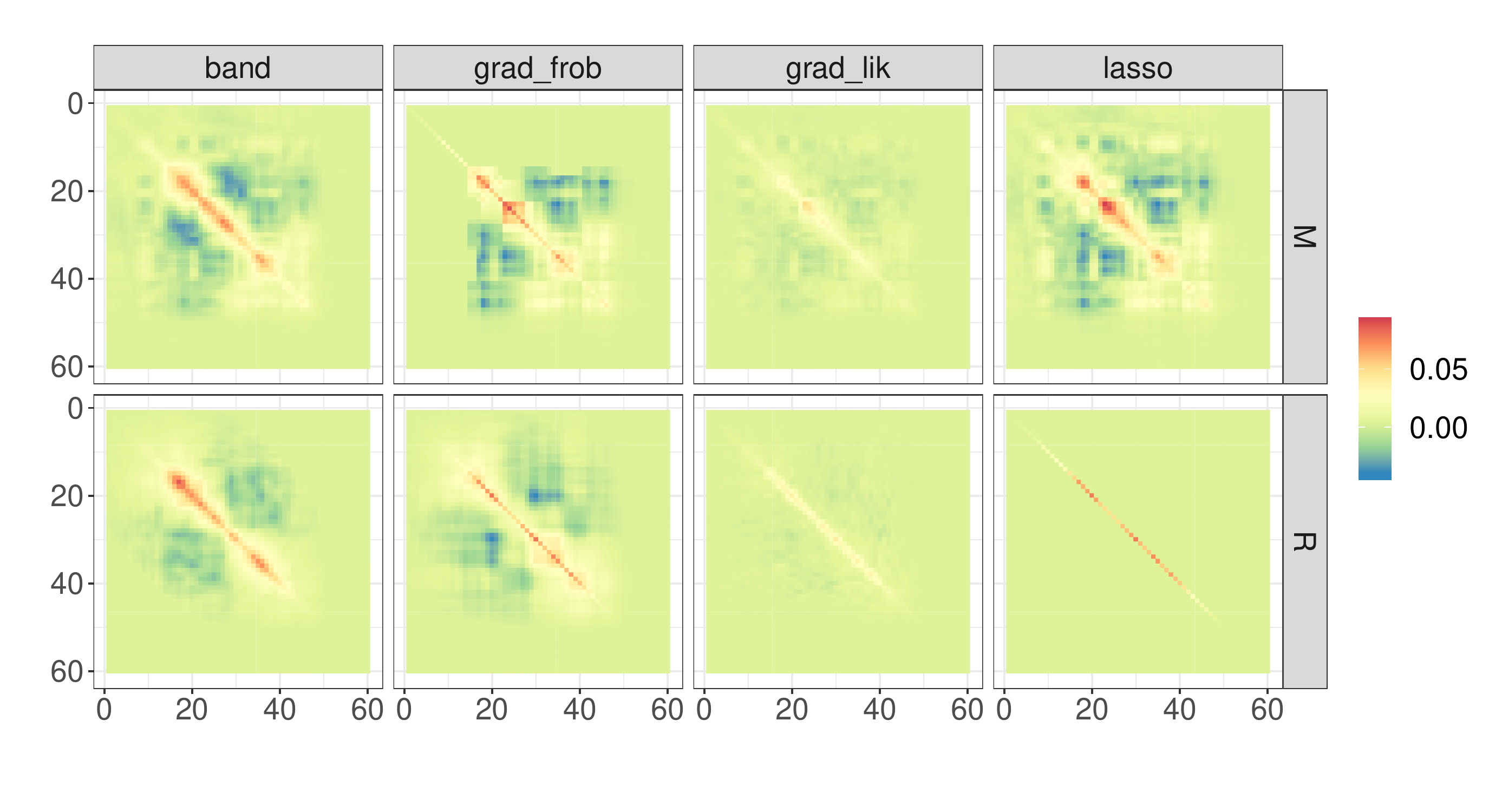}
	\caption{Heatmaps of the covariance matrix of rock and mine samples. M: Mines; R: Rocks.}
	\label{fig:sonar:covs}
\end{figure}

\begin{figure}[h!]
	\centering
	\includegraphics[width = \linewidth]{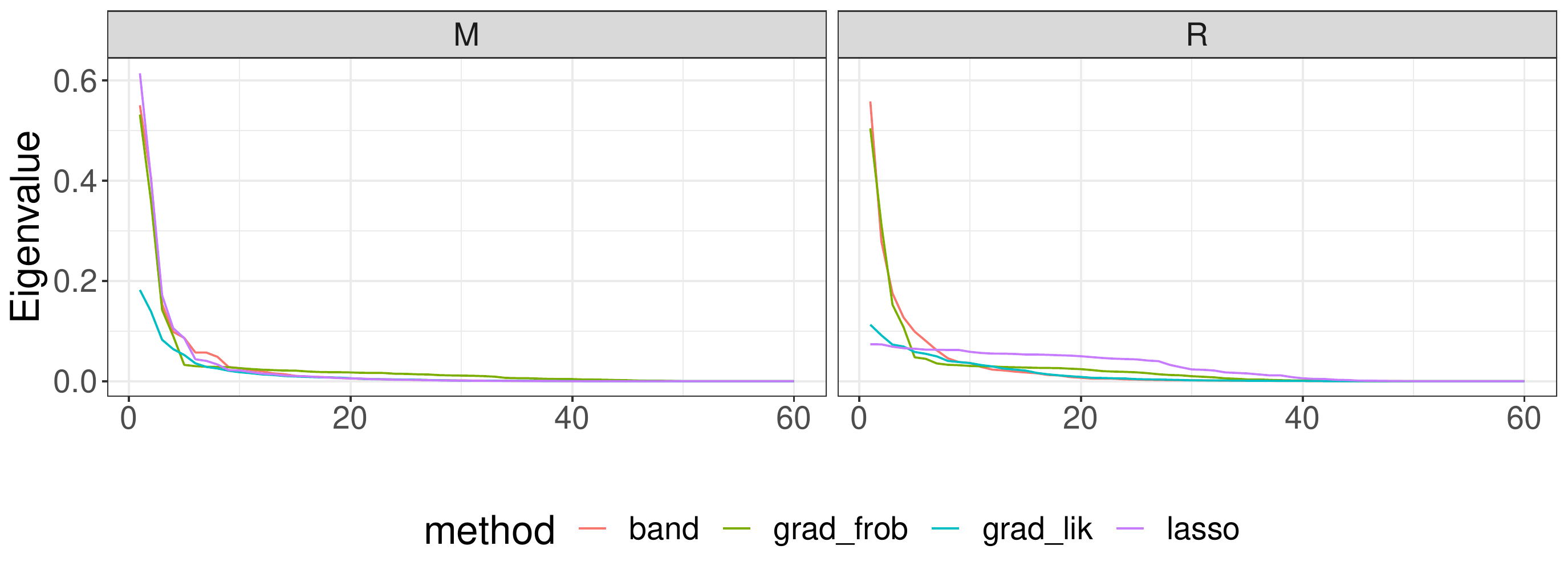}
	\caption{Scree plot of the covariance matrices of rock and mine samples. M: Mines; R: Rocks.}
	\label{fig:sonar:scree}
\end{figure}

\begin{figure}[h!]
	\centering
	\includegraphics[width = \linewidth]{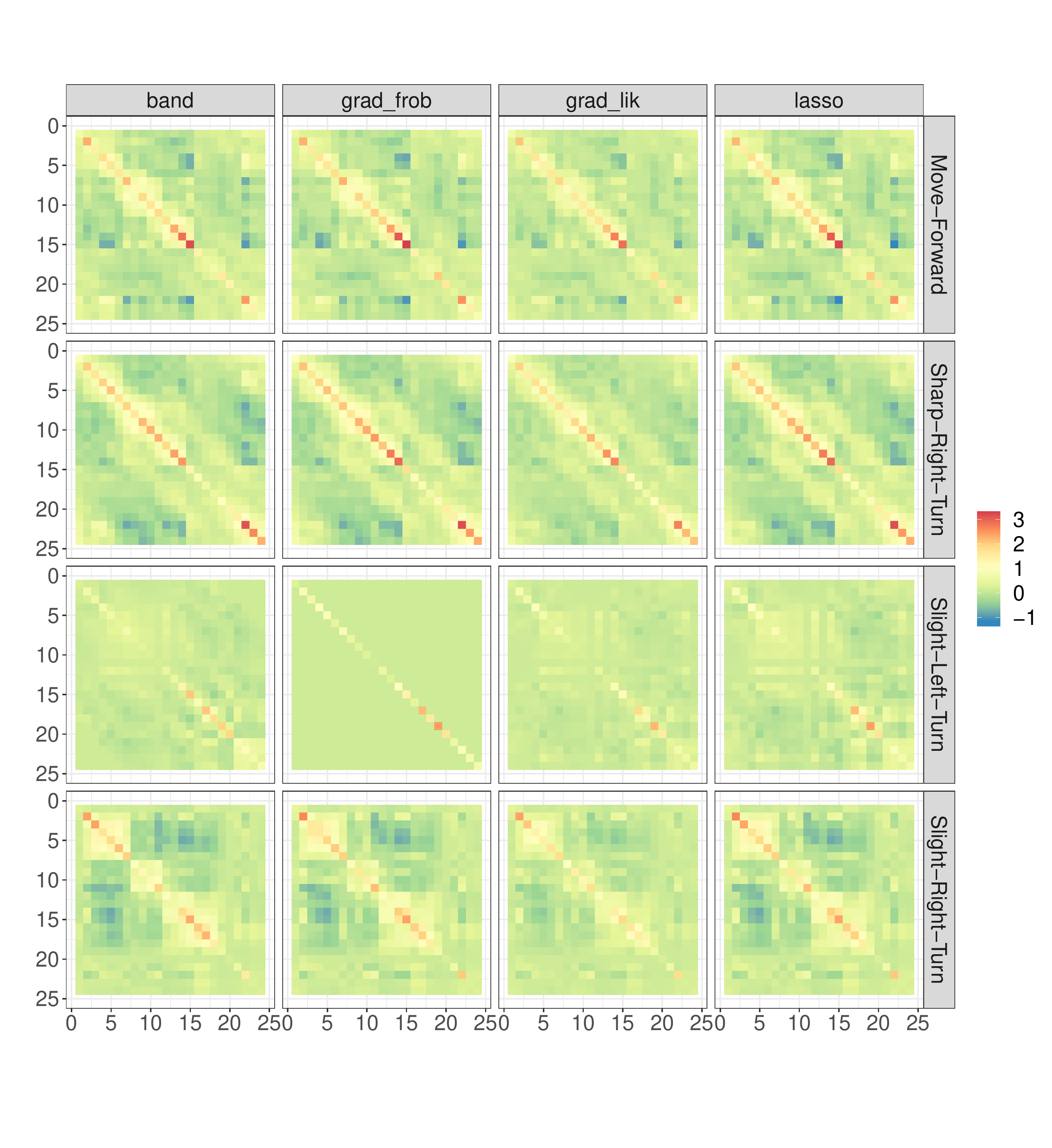}
	\caption{Heatmaps of the covariance matrix of the wall-following robot \rw{navigation} samples.}
	\label{fig:robot:covs}
\end{figure}

\begin{figure}[h!]
	\centering
	\includegraphics[width = \linewidth]{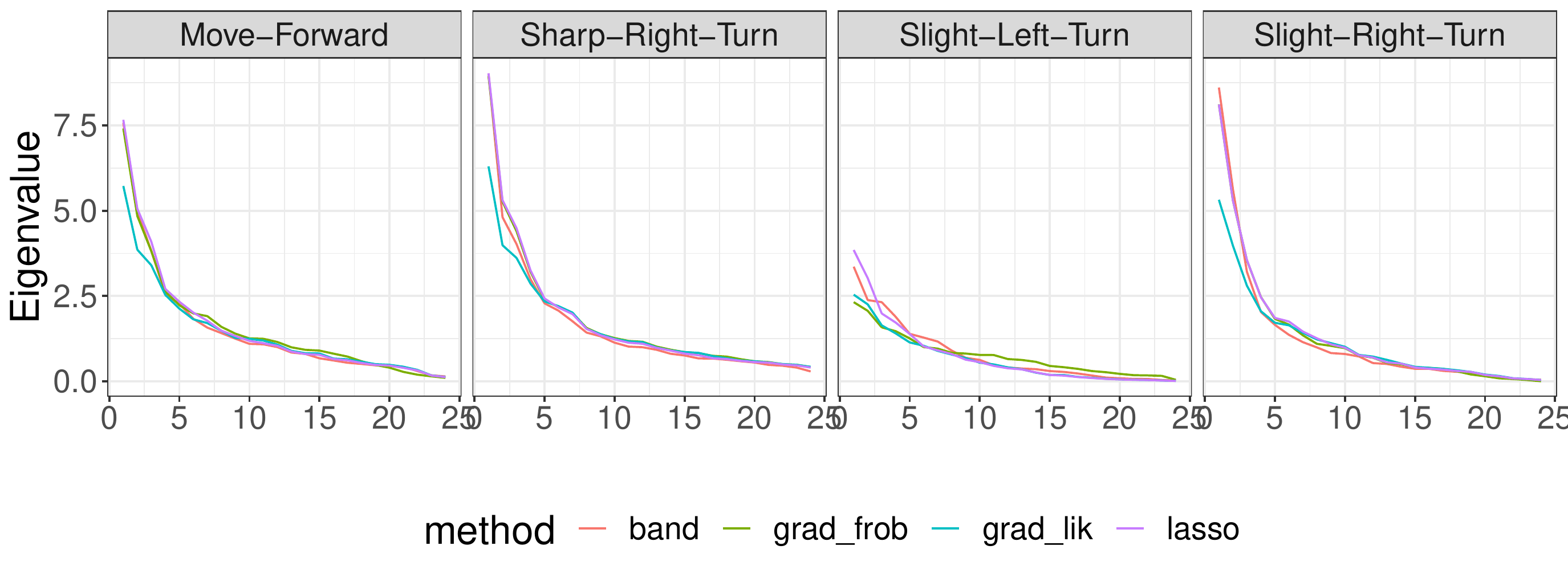}
	\caption{Scree plot of the covariance matrices of the wall-following robot \rw{navigation} samples.}
	\label{fig:robot:scree}
\end{figure}

\bibliographystyle{IEEEtran}

\bibliography{biblio}

\begin{thebibliography}{10}
\providecommand{\url}[1]{#1}
\csname url@samestyle\endcsname
\providecommand{\newblock}{\relax}
\providecommand{\bibinfo}[2]{#2}
\providecommand{\BIBentrySTDinterwordspacing}{\spaceskip=0pt\relax}
\providecommand{\BIBentryALTinterwordstretchfactor}{4}
\providecommand{\BIBentryALTinterwordspacing}{\spaceskip=\fontdimen2\font plus
\BIBentryALTinterwordstretchfactor\fontdimen3\font minus
  \fontdimen4\font\relax}
\providecommand{\BIBforeignlanguage}[2]{{%
\expandafter\ifx\csname l@#1\endcsname\relax
\typeout{** WARNING: IEEEtran.bst: No hyphenation pattern has been}%
\typeout{** loaded for the language `#1'. Using the pattern for}%
\typeout{** the default language instead.}%
\else
\language=\csname l@#1\endcsname
\fi
#2}}
\providecommand{\BIBdecl}{\relax}
\BIBdecl

\bibitem{buehlmann2011}
P.~B\"uhlmann and S.~van~de Geer, \emph{Statistics for High-Dimensional Data:
  Methods, Theory and Applications}.\hskip 1em plus 0.5em minus 0.4em\relax
  Springer, 2011.

\bibitem{maathuis2018}
M.~Maathuis, M.~Drton, S.~Lauritzen, and M.~Wainwright, Eds., \emph{Handbook of
  Graphical Models}.\hskip 1em plus 0.5em minus 0.4em\relax CRC Press, 2018.

\bibitem{li2020}
T.~Li, C.~Qian, E.~Levina, and J.~Zhu, ``High-dimensional {G}aussian graphical
  models on network-linked data.'' \emph{Journal of Machine Learning Research},
  vol.~21, no.~74, pp. 1--45, 2020.

\bibitem{dempster1972}
A.~P. Dempster, ``Covariance selection,'' \emph{Biometrics}, vol.~28, no.~1,
  pp. 157--175, 1972.

\bibitem{yuan2007}
M.~Yuan and Y.~Lin, ``Model selection and estimation in the {Gaussian}
  graphical model,'' \emph{Biometrika}, vol.~94, no.~1, pp. 19--35, 2007.

\bibitem{wright1934}
S.~Wright, ``The method of path coefficients,'' \emph{Ann. Math. Stat.},
  vol.~5, no.~3, pp. 161--215, 1934.

\bibitem{wermuth1980}
N.~Wermuth, ``Linear recursive equations, covariance selection, and path
  analysis,'' \emph{J. Am. Stat. Assoc.}, vol.~75, no. 372, pp. 963--972, 1980.

\bibitem{vandegeer2013}
S.~van~de Geer and P.~Bühlmann, ``$\ell_{0}$-penalized maximum likelihood for
  sparse directed acyclic graphs,'' \emph{Ann. Stat.}, vol.~41, no.~2, pp.
  536--567, 2013.

\bibitem{pourahmadi1999}
M.~Pourahmadi, ``Joint mean-covariance models with applications to longitudinal
  data: Unconstrained parameterisation,'' \emph{Biometrika}, vol.~86, no.~3,
  pp. 677--690, 1999.

\bibitem{daspermont2008}
A.~d'Aspremont, O.~Banerjee, and L.~El~Ghaoui, ``First-order methods for sparse
  covariance selection,'' \emph{SIAM J. Matrix Anal. Appl.}, vol.~30, no.~1,
  pp. 56--66, 2008.

\bibitem{friedman2008}
J.~Friedman, T.~Hastie, and R.~Tibshirani, ``Sparse inverse covariance
  estimation with the graphical lasso,'' \emph{Biostatistics}, vol.~9, no.~3,
  pp. 432--441, 2008.

\bibitem{rothman2008}
A.~J. Rothman, P.~J. Bickel, E.~Levina, and J.~Zhu, ``Sparse permutation
  invariant covariance estimation,'' \emph{Electron. J. Stat.}, vol.~2, pp.
  494--515, 2008.

\bibitem{cordoba2020b}
I.~Córdoba, C.~Bielza, and P.~Larrañaga, ``A review of {G}aussian {M}arkov
  models for conditional independence,'' \emph{J. Stat. Plan. Inference}, vol.
  206, pp. 127--144, 2020.

\bibitem{kauermann1996}
G.~Kauermann, ``On a dualization of graphical {G}aussian models,'' \emph{Scand.
  J. Stat.}, vol.~23, no.~1, pp. 105--116, 1996.

\bibitem{cox1993}
D.~R. Cox and N.~Wermuth, ``Linear dependencies represented by chain graphs,''
  \emph{Stat. Sci.}, vol.~8, no.~3, pp. 204--218, 1993.

\bibitem{chaudhuri2007}
S.~Chaudhuri, M.~Drton, and T.~S. Richardson, ``Estimation of a covariance
  matrix with zeros,'' \emph{Biometrika}, vol.~94, no.~1, pp. 199--216, 2007.

\bibitem{bien2011}
J.~Bien and R.~J. Tibshirani, ``Sparse estimation of a covariance matrix,''
  \emph{Biometrika}, vol.~98, no.~4, pp. 807--820, 2011.

\bibitem{wermuth2006}
N.~Wermuth, D.~Cox, and G.~M. Marchetti, ``Covariance chains,''
  \emph{Bernoulli}, vol.~12, no.~5, pp. 841--862, 2006.

\bibitem{rothman2010}
A.~J. Rothman, E.~Levina, and J.~Zhu, ``A new approach to {C}holesky-based
  covariance regularization in high dimensions,'' \emph{Biometrika}, vol.~97,
  no.~3, pp. 539--550, 2010.

\bibitem{das2019}
A.~Das, D.~Sexton, C.~Lainscsek, S.~S. Cash, and T.~J. Sejnowski,
  ``Characterizing brain connectivity from human electrocorticography
  recordings with unobserved inputs during epileptic seizures,'' \emph{Neural
  Comput.}, vol.~31, no.~7, pp. 1271--1326, 2019.

\bibitem{varando2020}
G.~Varando and N.~R. Hansen, ``Graphical continuous {L}yapunov models,'' in
  \emph{Proceedings of the 36th conference on Uncertainty in Artificial
  Intelligence}.\hskip 1em plus 0.5em minus 0.4em\relax Accepted, 2020,
  arXiv:2005.10483.

\bibitem{drton2018}
M.~Drton, ``Algebraic problems in structural equation modeling,'' in \emph{The
  50th Anniversary of Gröbner Bases}.\hskip 1em plus 0.5em minus 0.4em\relax
  Mathematical Society of Japan, 2018, pp. 35--86.

\bibitem{dempster1969}
A.~P. Dempster, \emph{Elements of Continuous Multivariate Analysis}.\hskip 1em
  plus 0.5em minus 0.4em\relax Adisson-Wesley, 1969.

\bibitem{beaton1964}
A.~E. Beaton, ``The use of special matrix operators in statistical calculus,''
  \emph{ETS Research Bulletin Series}, vol. 1964, no.~2, pp. i--222, 1964.

\bibitem{ye2020}
Q.~{Ye}, A.~{Amini}, and Q.~{Zhou}, ``Optimizing regularized cholesky score for
  order-based learning of bayesian networks,'' \emph{IEEE Transactions on
  Pattern Analysis and Machine Intelligence}, vol. Available online, pp. 1--1,
  2020.

\bibitem{dawid1979}
A.~P. Dawid, ``Conditional independence in statistical theory,'' \emph{J. R.
  Stat. Soc. Ser. B Stat. Methodol.}, vol.~41, no.~1, pp. 1--31, 1979.

\bibitem{anderson2003}
T.~W. Anderson, \emph{An Introduction to Multivariate Statistical Analysis},
  3rd~ed.\hskip 1em plus 0.5em minus 0.4em\relax John Wiley \& Sons, 2003.

\bibitem{chandrasekaran2012}
V.~Chandrasekaran, P.~A. Parrilo, and A.~S. Willsky, ``Latent variable
  graphical model selection via convex optimization,'' \emph{Ann. Statist.},
  vol.~40, no.~4, pp. 1935--1967, 08 2012.

\bibitem{yatsenko2015}
D.~Yatsenko, K.~Josić, A.~S. Ecker, E.~Froudarakis, R.~J. Cotton, and A.~S.
  Tolias, ``Improved estimation and interpretation of correlations in neural
  circuits,'' \emph{PLOS Computational Biology}, vol.~11, no.~3, pp. 1--28,
  2015.

\bibitem{zorzi2016}
M.~{Zorzi} and R.~{Sepulchre}, ``{AR} identification of latent-variable
  graphical models,'' \emph{IEEE Transactions on Automatic Control}, vol.~61,
  no.~9, pp. 2327--2340, 2016.

\bibitem{basu2019}
S.~{Basu}, X.~{Li}, and G.~{Michailidis}, ``Low rank and structured modeling of
  high-dimensional vector autoregressions,'' \emph{IEEE Transactions on Signal
  Processing}, vol.~67, no.~5, pp. 1207--1222, 2019.

\bibitem{levina2008}
E.~Levina, A.~Rothman, and J.~Zhu, ``Sparse estimation of large covariance
  matrices via a nested {L}asso penalty,'' \emph{Ann. Appl. Stat.}, vol.~2,
  no.~1, pp. 245--263, 2008.

\bibitem{zheng2018}
X.~Zheng, B.~Aragam, P.~Ravikumar, and E.~P. Xing, ``{DAGs} with {NO TEARS}:
  {C}ontinuous optimization for structure learning,'' in \emph{Proceedings of
  the 32nd International Conference on Neural Information Processing
  Systems}.\hskip 1em plus 0.5em minus 0.4em\relax Curran Associates Inc.,
  2018, p. 9492–9503.

\bibitem{bach2012}
F.~Bach, R.~Jenatton, J.~Mairal, and G.~Obozinski, ``Optimization with
  sparsity-inducing penalties,'' \emph{Found. Trends Mach. Learn.}, vol.~4,
  no.~1, p. 1–106, 2012.

\bibitem{boyd2004}
S.~Boyd and L.~Vandenberghe, \emph{Convex Optimization}.\hskip 1em plus 0.5em
  minus 0.4em\relax Cambridge University Press, 2004.

\bibitem{rcore}
\BIBentryALTinterwordspacing
{R Core Team}, \emph{R: {A} Language and Environment for Statistical
  Computing}, R Foundation for Statistical Computing, 2020. [Online].
  Available: \url{https://www.R-project.org/}
\BIBentrySTDinterwordspacing

\bibitem{cordoba2020}
I.~Córdoba, G.~Varando, C.~Bielza, and P.~Larrañaga, ``On generating random
  {G}aussian graphical models,'' \emph{International Journal of Approximate
  Reasoning}, vol. 125, pp. 240--250, 2020.

\bibitem{dua2020}
\BIBentryALTinterwordspacing
D.~Dua and C.~Graff, ``{UCI} machine learning repository,'' 2020. [Online].
  Available: \url{http://archive.ics.uci.edu/ml}
\BIBentrySTDinterwordspacing

\bibitem{petersen2008matrix}
K.~Petersen and M.~Pedersen, \emph{The Matrix Cookbook}.\hskip 1em plus 0.5em
  minus 0.4em\relax Technical University of Denmark, 2008.

\end{thebibliography}

\end{document}